\documentclass[twoside]{article}

\usepackage[preprint]{aistats2027}

\usepackage{microtype}
\usepackage{graphicx}
\usepackage{subcaption}
\usepackage{booktabs} 
\usepackage{hyperref}
\hypersetup{colorlinks=true, linkcolor=blue, citecolor=violet, urlcolor=blue}

\usepackage{lipsum}  
\usepackage{placeins}
\usepackage{float}

\usepackage{array}
\usepackage[table]{xcolor} 

\usepackage{algorithm}
\usepackage{algpseudocode}
\usepackage{amsmath}

\usepackage{makecell}
\usepackage{float}
\usepackage[table]{xcolor}
\usepackage{colortbl}
\usepackage{array}
\usepackage{adjustbox}
\usepackage{subcaption}
\usepackage{placeins}

\newcommand{\cTwoST}[2]{#1\,{\tiny$\pm$#2}}

\usepackage{url}

\newcommand{\code}[1]{\path{#1}}

\usepackage{amsmath}
\usepackage{amssymb}
\usepackage{mathtools}
\usepackage{amsthm}
\usepackage{bm}   

\usepackage{booktabs}
\usepackage{multirow}
\usepackage{makecell}
\usepackage{xurl}
\usepackage{amsthm}

\newtheorem{proposition}{Proposition}

\usepackage[capitalize,noabbrev]{cleveref}

\usepackage[round]{natbib}

\begin{document}
\runningauthor{Nautiyal, Ju, Ernfors, Hagland, Holma, S\"oderholm, Hellander, Singh}
\addtocontents{toc}{\protect\setcounter{tocdepth}{-1}}
\raggedbottom

\twocolumn[

\aistatstitle{OneFlowSBI: One Model, Many Queries for Simulation-Based Inference}

\aistatsauthor{
Mayank Nautiyal$^{1,2}$ \And
Li Ju$^{1,2}$ \And
Melker Ernfors$^{1}$ \And
Klara Hagland$^{1}$ \And
Ville Holma$^{1}$ \AND
Maximilian Werk\"o S\"oderholm$^{1}$ \And
Andreas Hellander$^{1}$ \And
Prashant Singh$^{1,2}$
}

\aistatsaddress{
$^{1}$Department of Information Technology, Uppsala University, Sweden \\
$^{2}$Science for Life Laboratory, Uppsala University, Sweden
} ]



\begin{abstract}
Simulation-based inference (SBI) methods typically amortize a single query, usually the posterior, so each further conditional or marginal needs its own model. Existing generalists amortize many queries in one transformer with one token per variable, so their attention cost grows quadratically with the number of variables. We introduce OneFlowSBI, a masked flow matching framework that trains one network on the joint distribution of parameters and observations. A ternary mask marks each coordinate as observed, generated or absent, so any conditional or marginal is sampled by changing the mask. The mask is an input, rather than a tokenization, so the backbone can follow the data's structure and scale to many variables. On high-dimensional image denoising and shallow water inference, OneFlowSBI achieves lower posterior error than the baselines, samples far faster than a transformer-based generalist, and handles missing regions and partial observation windows without retraining. On most of the ten SBIBM tasks at the largest simulation budget, it matches or exceeds neural posterior estimation and is comparable to transformer-based generalists, while also providing the likelihood and marginals.
\end{abstract}

\section{Introduction}\label{sec:intro}
Scientific simulators encode physical laws and mechanistic assumptions to generate synthetic observations from interpretable parameters~\citep{Lavin2021SimulationIT}. Once a simulator is constructed, scientists use it to study the underlying system: they may want to infer parameters from measured observations, predict observations under specified parameter settings~\citep{Radev23}, use known values of some parameters to constrain others, or determine what can still be inferred when measurements are incomplete~\citep{wang2024missing}. The relevant inference query therefore depends on which quantities are known, which are sought, and which are unavailable. Although these questions arise from the same scientific model, they correspond to different conditionals or marginals of one joint distribution of parameters and observations.

For most realistic simulators, however, these distributions are typically not available in closed form: the likelihood is defined only implicitly, through the unobserved latent variables of the simulation, so classical likelihood-based inference is impractical~\citep{Cranmer20}. Simulation-based inference (SBI) instead learns from simulated pairs $(\bm{\theta},\mathbf{y}) \sim p(\bm{\theta})\,p(\mathbf{y} \mid \bm{\theta})$, where the prior $p(\bm{\theta})$ encodes domain knowledge and the likelihood is accessed only by running the simulator~\citep{sisson2018handbook}. Neural networks trained on these pairs then provide surrogates for the required distributions, typically under a limited simulation budget.

Most SBI methods are specialists that learn and amortize one of these distributions, typically the posterior~\citep{papamakarios2016fast,Greenberg19,fmpe}. Each additional question therefore needs its own model, even though all these questions concern the same simulator. In recent years, generalist methods such as Simformer~\citep{gloecklerall} and the Amortized Conditioning Engine (ACE)~\citep{chang2025amortized} have been proposed to solve this, amortizing arbitrary conditionals within one transformer. Because they assign one token per variable or data point, their attention cost grows quadratically with the number of variables. When observations are high-dimensional, these methods therefore either become computationally prohibitive or must compress the observations into a few summary tokens, which rules out conditioning on individual values (Section~\ref{sec:related}).

To address this, we introduce \textit{OneFlowSBI}, a masked flow matching framework that learns the joint distribution of parameters and observations with a single network. A ternary mask marks each coordinate as observed, generated or absent, mirroring which quantities are known, sought or unavailable: observed coordinates are held fixed, generated ones are transported from noise by the learned flow, and absent ones are marginalized out, so any conditional or marginal is sampled by changing the mask. Rather than outperforming specialists on the one query they are trained for, we aim for one model that answers every query with competitive accuracy and scales to high-dimensional problems, where token-based generalists become costly.

\paragraph{Contributions.}
\textbf{(i)} We formulate generalist SBI as masked flow matching on the joint state, with an absent role that marginalizes coordinates without generating them. Following standard flow matching arguments, we show that, at the population optimum and under standard regularity conditions, every mask with positive training probability recovers its target distribution (Appendix~\ref{app:proof}).

\textbf{(ii)} We decouple the query from the architecture: since the mask is an input rather than a tokenization, the backbone can follow the structure of the data, from an MLP for vectors to a convolutional network for images and spatiotemporal fields. The model thus conditions on individual values without compressing the observations, and its cost per network evaluation can grow only linearly in the number of variables $d$, against $\mathcal{O}(d^2)$ per attention layer for transformer-based generalist methods.

\textbf{(iii)} We study the advantages and limitations of this design on problems ranging from standard benchmarks to high-dimensional applications in imaging and geophysical fluid dynamics with up to $20{,}300$ variables, assessing the accuracy and calibration of the learned distributions and their computational cost against specialist and generalist baselines. Ablations quantify the amortization gap over queries and the sensitivity to the mask distribution.

\paragraph{Outline.}
Section~\ref{sec:related} reviews related work, and Section~\ref{sec:method} presents the method. Sections~\ref{sec:experiments} and~\ref{sec:ablation} report experiments and ablations, Section~\ref{sec:limit} discusses limitations, and Section~\ref{sec:conclusion} concludes.

\section{Related Work}
\label{sec:related}

\paragraph{Specialist SBI.} Most SBI methods~\citep{Cranmer20, DeistlerBoelts_simulationbased_2025} are built around the inverse problem of posterior estimation: they target a single distribution and amortize inference over observations, so that one trained network infers parameters from any new observation. They differ mainly in the underlying generative model: neural posterior estimation (NPE) typically uses a normalizing flow~\citep{papamakarios2016fast,NIPS2017_addfa9b7,Greenberg19}, while later methods extend the approach to other families of generative models, such as generative adversarial networks~\citep{Ramesh22}, diffusion models~\citep{geffner2023compositional} and flow matching~\citep{fmpe,zhai2025conditional}. A second family learns the likelihood, as in neural likelihood estimation~\citep{pmlr-v89-papamakarios19a}, or the likelihood-to-evidence ratio, as in neural ratio estimation~\citep{hermans2020likelihood}. Both can be combined with different priors, but sampling the posterior then requires an MCMC run for each observation. When observations are high-dimensional, these methods typically reduce them to a learned summary~\citep{radev2020bayesflow}, which serves one query but discards the individual values that other queries need. Each specialist therefore answers a single query, and every other distribution requires another model or inference procedure~\citep{Cranmer20}.

\paragraph{Missing and partial observations.} Posterior estimators trained on complete data cannot be applied directly when parts of the observation are missing. Existing remedies train on data whose missing entries are filled in and flagged by a binary mask~\citep{wang2024missing}, or learn imputation and inference jointly~\citep{verma2025robust}, but still target the posterior alone. Outside SBI, CFMI imputes missing values with one flow matching model shared across missingness patterns~\citep{simkus2025cfmi}. \textit{OneFlowSBI} instead marks missing coordinates as absent, which marginalizes them within the same model that answers every other query.

\paragraph{Generalist SBI.} JANA~\citep{Radev23} and SNPLA~\citep{wiqvist2021sequential} train separate posterior and likelihood networks. Simformer~\citep{gloecklerall} and ACE~\citep{chang2025amortized} instead amortize conditionals within one transformer. Simformer, a score-based diffusion model with one token per variable, has an attention cost of $\mathcal{O}(d^2)$ per layer. Its authors recommend embedding high-dimensional observations into a few tokens, which limits conditioning to whole groups of values. ACE assigns one token to each of $N$ context and $M$ target points, at an attention cost of $\mathcal{O}(N^2 + NM)$, and predicts each target coordinate as an independent one-dimensional mixture, so joint samples require one autoregressive pass per coordinate~\citep[App.~B.4]{chang2025amortized}. Because its training protocol requires at least one context token, unconditional distributions such as $p(\bm{\theta})$, $p(\mathbf{y})$ or $p(\bm{\theta}, \mathbf{y})$ lie outside it. CogFormer uses flow matching to amortize the posterior across families of models and experimental designs, with a mask that selects which parameters each model uses~\citep{huang2026cogformer}. \textit{OneFlowSBI} extends Simformer's condition mask with an absent role and supplies it as an input rather than through tokens, so its cost is set by the backbone chosen for the data rather than by the number of tokens; unlike ACE, it samples joint conditionals in a single ODE solve and also yields marginals and the joint. Appendix~\ref{sec:faq} answers common questions about these design choices.

\section{Method}\label{sec:method}

In the standard SBI setting, parameters
$\bm{\theta}\in\mathbb R^{d_\theta}$ are drawn from a user-specified prior $p(\bm{\theta})$, and corresponding observations
$\mathbf y\in\mathbb R^{d_y}$ are generated by a simulator
$p(\mathbf y\mid\bm{\theta})$ whose likelihood is unavailable or intractable to evaluate. Our goal is to learn a unified model that supports multiple probabilistic queries associated with the joint distribution $p(\bm{\theta},\mathbf y)$, including its posterior, likelihood, and marginal distributions.
\textit{OneFlowSBI} achieves this with a single mask-conditioned
flow matching model over the joint state
$\mathbf z=[\bm{\theta};\mathbf y]\in\mathbb R^{d_\theta+d_y}$.
The shared vector field is conditioned on a mask that specifies the role of each coordinate and thereby determines the probabilistic query represented by the flow.

\paragraph{Masked probability path.}
To represent different query types within this shared model, we introduce a
ternary mask $\mathbf m\in\{-1,0,1\}^{d}$, where
$d=d_\theta+d_y$ is the dimension of the augmented state, and $0$, $1$, and
$-1$ denote generated, observed, and absent coordinates, respectively.
Generated coordinates are transported by the flow, observed coordinates remain
fixed at their conditioning values, and absent coordinates are excluded
from both conditioning and generation. The absent state extends binary observed/generated
masking to marginal queries; for example, marking $\bm{\theta}$ as absent while
generating $\mathbf y$ defines a flow targeting the observation marginal
$p(\mathbf y)$. Appendix~\ref{faq:absent} discusses what the absent role adds over a binary mask.

From $\mathbf m$, we construct binary masks
$\mathbf m^{\mathrm{obs}},\mathbf m^{\mathrm{gen}}\in\{0,1\}^{d}$ selecting
the observed and generated coordinates, respectively. Absent coordinates are
selected by neither mask. Following the linear probability path of flow
matching~\citep{lipman2023flow}, we define its masked
counterpart over the augmented state. Let
$\mathbf z_0\sim\mathcal N(\mathbf 0,\mathbf I_d)$ and
$\mathbf z_1\sim p(\bm{\theta})p(\mathbf y\mid\bm{\theta})$ denote independent,
unmasked base and target samples. For $t\in[0,1]$, the masked interpolant is
\begin{equation}
\mathbf z_t
=
\mathbf m^{\mathrm{obs}}\odot\mathbf z_1
+
\mathbf m^{\mathrm{gen}}\odot
\left[(1-t)\mathbf z_0+t\mathbf z_1\right].
\label{eq:masked-interpolant}
\end{equation}
This applies the linear interpolant only to generated coordinates, while
observed coordinates remain fixed at their conditioning values and absent
coordinates remain outside the active probability path (Figure~\ref{fig:oneflowsbi-overview}). The corresponding
conditional target velocity is
\begin{equation}
\mathbf v_t(\mathbf z_t\mid\mathbf z_1,\mathbf m)
=
\mathbf m^{\mathrm{gen}}\odot
(\mathbf z_1-\mathbf z_0).
\label{eq:masked-target-vel}
\end{equation}

A detailed derivation of the masked probability path and its specialization
to joint, posterior, likelihood, marginal, mixed-conditional, and
missing-observation queries is provided in
Appendix~\ref{sec:masked-path-derivation}.

\begin{figure}[t]
    \centering
    \includegraphics[width=\columnwidth]{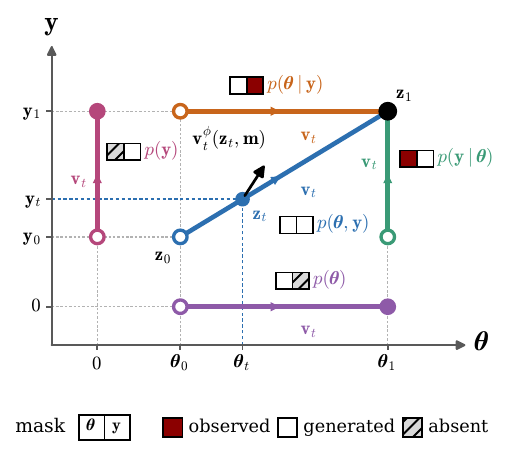}
    \caption{
    \textbf{OneFlowSBI overview.}
    Training paths under the masked linear interpolant for joint,
    posterior, likelihood, and marginal queries. Trajectories induced
    by the learned velocity field need not be straight. Absent
    coordinates use zero placeholders, not conditioning values.
    }
    \label{fig:oneflowsbi-overview}
\end{figure}

\paragraph{Masked flow matching objective.}
We parameterize a shared time-dependent vector field
$\mathbf v_t^\phi(\mathbf z_t,\mathbf m)$, explicitly conditioned on the
query mask, and match the target velocity only on coordinates designated as
generated. We define the masked velocity residual
\begin{equation}
\mathbf r_t^\phi
=
\mathbf m^{\mathrm{gen}}\odot
\left[
\mathbf v_t^\phi(\mathbf z_t,\mathbf m)
-
(\mathbf z_1-\mathbf z_0)
\right].
\label{eq:masked-residual}
\end{equation}
The resulting \textit{OneFlowSBI} objective is
\begin{equation}
\mathcal L_{\mathrm{OneFlowSBI}}(\phi)
=
\mathbb E_{t,\mathbf z_0,\mathbf z_1,\mathbf m}
\left[
\frac{1}{n_{\mathrm{gen}}}
\left\|
\mathbf r_t^\phi
\right\|_2^2
\right],
\label{eq:oneflowsbi-loss}
\end{equation}
where $t\sim\mathcal U[0,1]$, $\mathbf m\sim p(\mathbf m)$ is the mask distribution described below, and $n_{\mathrm{gen}}$ denotes the number of
generated coordinates under the sampled mask. Under standard regularity conditions, the population minimizer of
Eq.~\ref{eq:oneflowsbi-loss} generates, for every mask with positive training
probability, the distribution specified by that mask, with absent coordinates
marginalized out (Appendix~\ref{app:proof}).

\paragraph{Training mask distribution.}
Training the masked flow matching objective in Eq.~\eqref{eq:oneflowsbi-loss}
requires a distribution over query masks $p(\mathbf m)$. We construct
$p(\mathbf m)$ to cover the principal SBI queries considered in this work,
including joint, posterior, likelihood, and marginal estimation, together
with partially conditioned configurations. The corresponding query masks are
summarized in Table~\ref{tab:masking_tasks}. We use $15\%$ joint, $20\%$ posterior, $20\%$ likelihood, $15\%$ $p(\mathbf y)$, $10\%$ $p(\bm{\theta})$ and $20\%$ partial masks. In a partial mask, each coordinate is assigned a random role (Appendix~\ref{app:sbi_hyperparams}).

\begin{table}[h]
\centering
\caption{\textbf{Masks for SBI queries.}
For mixed conditionals, $G$ and $O$ partition the coordinates
into generated and observed subsets.}
\label{tab:masking_tasks}
\small
\setlength{\tabcolsep}{3pt}
\renewcommand{\arraystretch}{1}
\resizebox{\columnwidth}{!}{%
\begin{tabular}{@{}llc@{}}
\toprule
Task & Target density & Mask $\mathbf{m}$ \\
\midrule
Joint
& $p(\boldsymbol{\theta},\mathbf{y})$
& $[\mathbf{0}_{d_\theta};\mathbf{0}_{d_y}]$ \\
Posterior
& $p(\boldsymbol{\theta}\mid\mathbf{y})$
& $[\mathbf{0}_{d_\theta};\mathbf{1}_{d_y}]$ \\
Likelihood
& $p(\mathbf{y}\mid\boldsymbol{\theta})$
& $[\mathbf{1}_{d_\theta};\mathbf{0}_{d_y}]$ \\
Marginal
& $p(\mathbf{y})$
& $[-\mathbf{1}_{d_\theta};\mathbf{0}_{d_y}]$ \\
Marginal
& $p(\boldsymbol{\theta})$
& $[\mathbf{0}_{d_\theta};-\mathbf{1}_{d_y}]$ \\
Mixed conditional
& $p(\mathbf{z}_G\mid\mathbf{z}_O)$
& $\mathbf{m}_G=0,\;\mathbf{m}_O=1$ \\
\bottomrule
\end{tabular}%
}
\end{table}

\paragraph{Sampling via masking.}
For a fixed query mask $\mathbf m$ and conditioning values
$\mathbf z_{\mathrm{obs}}$, we initialize the trajectory at $t=0$ with
$\mathbf z_0=\mathbf m^{\mathrm{obs}}\odot\mathbf z_{\mathrm{obs}}
+\mathbf m^{\mathrm{gen}}\odot\bm{\varepsilon}$, where
$\bm{\varepsilon}\sim\mathcal N(\mathbf 0,\mathbf I)$, and integrate the
learned ODE forward from $t=0$ to $t=1$,
\begin{equation}
\frac{d\mathbf z_t}{dt}
=
\mathbf m^{\mathrm{gen}}\odot
\mathbf v_t^\phi(\mathbf z_t,\mathbf m).
\label{eq:oneflowsbi-ode}
\end{equation}
Let $G$, $O$ and $A$ index the generated, observed and absent coordinates.
The generated coordinates of the terminal state $\mathbf z_1$ are
approximate samples from $q_\phi^{\mathbf m}(\mathbf z_G\mid\mathbf z_O)$,
the learned approximation to $p(\mathbf z_G\mid\mathbf z_O)$, in which the
absent coordinates $\mathbf z_A$ are marginalized out. Sampling procedures
for the individual SBI distributions are given in
Appendix~\ref{sec:sampling-sbi-distributions}.

\begin{figure}[t]
  \centering
  \includegraphics[width=\columnwidth]{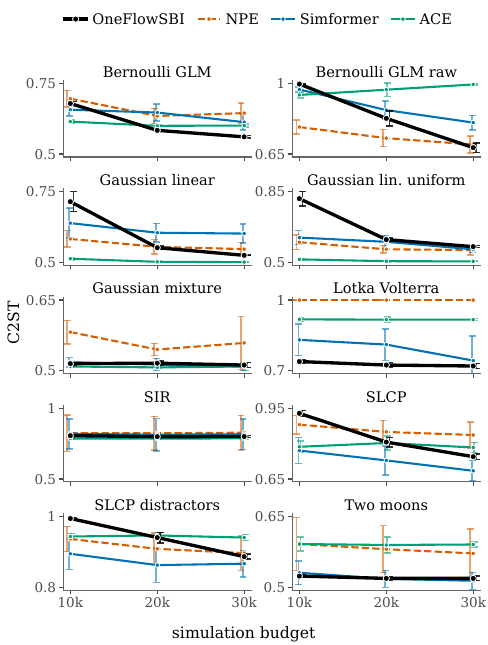}
  \caption{Posterior C2ST versus simulation budget (lower is better, 0.5 is
  optimal). Bars show $\pm$ one standard deviation: over five training runs and over the ten reference observations.}
  \label{fig:c2st_plot_main}
\end{figure}

\section{Experiments}\label{sec:experiments}
We order the experiments by the dimension $d = d_\theta + d_y$ of the joint state $\mathbf{z} = [\bm{\theta}; \mathbf{y}]$, to study how accuracy and cost change as problems grow. We first evaluate ten tasks of the simulation-based inference benchmark (SBIBM)~\citep{lueckmann2021benchmarking} with $d$ from $4$ to $110$, and then study three problems in depth: Lotka--Volterra ($d = 24$), Bayesian denoising on Fashion MNIST~\citep{Radev23} ($d = 1{,}568$) and depth inference from shallow water waves~\citep{Ramesh22} ($d = 20{,}300$). We compare with NPE, a posterior specialist, and with the transformer-based generalists Simformer and ACE; in the case studies, we use only Simformer, which also samples the joint and marginals that lie outside the training protocol of ACE (Appendix~\ref{app:hyperparam_case}). 

Since simulators in scientific applications are often expensive, we train all methods with at most $30{,}000$ simulations: the full budget in the case studies, and budgets of $10{,}000$, $20{,}000$ and $30{,}000$ on SBIBM. Throughout, $\mathbf{y}_{\mathrm{obs}}$ denotes an observation simulated at a true parameter $\bm{\theta}^{*}$ drawn from the prior; SBIBM provides ten such reference pairs per task, each with samples from the true posterior. Where reference samples are available, accuracy is measured by the classifier two-sample test (C2ST)~\citep{Lopez16}, which is $0.5$ when generated and reference samples are indistinguishable. Calibration is assessed on new pairs: a calibrated conditional contains the truth in its credible region of level $q$ with probability $q$~\citep{talts2018validating}. Simulators are described in Appendix~\ref{app:task_desc}, and architectures, training and baseline configurations in Appendix~\ref{app:implementation}.

\subsection{Benchmark tasks}
The posterior solves the inverse problem and is the primary SBI query, so we first explore: \textit{how accurate is the posterior?} At the largest budget, \textit{OneFlowSBI} is clearly the most accurate method on Bernoulli GLM, and the most accurate or on par with the best on Bernoulli GLM raw, Lotka--Volterra, SIR, Gaussian mixture and Two moons. It is close behind the best method on Gaussian linear and SLCP distractors, and further behind on Gaussian linear uniform and SLCP, but at least as accurate as NPE on every task except Gaussian linear uniform (Figure~\ref{fig:c2st_plot_main}). It also gains the most from more simulations, trailing at the smallest budget on several tasks but closing the gap or overtaking the baselines as the budget grows (Appendix~\ref{app:c2st_tables}).

\textit{How well are the other distributions learned?} The likelihood, prior and prior predictive approximate the one-dimensional marginals of their reference samples, including multimodal structure (Appendix~\ref{app:sbi_plots}). Considering C2ST, \textit{OneFlowSBI} is as accurate as Simformer or more accurate for the likelihood on every task, clearly so on Lotka--Volterra and SLCP, and for the prior predictive on nine of ten tasks; both approximate the prior closely, except Simformer on Lotka--Volterra. For both methods, the likelihood is the hardest query where observations are binary or contain many uninformative dimensions (Appendix~\ref{app:query_c2st}).

\textit{Are they calibrated, and at what cost?} Appendix~\ref{app:calibration} reports calibration over $1{,}000$ test pairs per task. The posterior is close to nominal on eight tasks, slightly overconfident on Bernoulli GLM raw and clearly so on SLCP distractors; the likelihood is well calibrated on seven tasks, conservative on Lotka--Volterra and SIR, and on SLCP distractors overconfident for some outputs and conservative for others. Appendix~\ref{app:train_sample_times} reports the cost: a model trains in under ten minutes on a GPU and draws $10{,}000$ samples of any query in about two seconds or less, faster than Simformer with its default sampler, and the gap grows with the number of variables, while ACE, which needs no solver, is faster on most tasks. On the standard benchmark, \textit{OneFlowSBI} is thus more accurate than the posterior specialist on most tasks, comparable to the transformer-based generalists on the posterior, and as accurate as Simformer or more accurate on the other queries.

\subsection{Case study 1: Lotka--Volterra}
The benchmark summarizes each learned distribution by a single score; we now look at the distributions themselves on Lotka--Volterra, a classical ecological model of prey and predator populations and a challenging SBIBM task. Four birth and death rates $\bm{\theta} = (\alpha, \beta, \gamma, \delta)$ generate noisy counts $\mathbf{y} \in \mathbb{R}^{20}$ of both populations at ten time points, and the posterior is among the most concentrated in the benchmark, with a standard deviation of about $5\%$ of the prior's. Because the simulator is stochastic, both directions of inference have a natural reference: the SBIBM reference posterior for $p(\bm{\theta} \mid \mathbf{y})$ and the simulator itself for $p(\mathbf{y} \mid \bm{\theta})$.

\textit{What does the posterior look like?} For the first reference observation, the posterior marginals concentrate on the same narrow region as the reference, with closely matching location and width, and a slightly wider marginal for $\alpha$ (Figure~\ref{fig:lv}a).

\textit{Does the same model reproduce the simulator?} At $\bm{\theta}^{*}$, the $90\%$ intervals of the generated trajectories follow those of the simulator through both population cycles, including the sharp peaks and the near zero counts between them, and are slightly wider only at the second predator peak (Figure~\ref{fig:lv}b).

\textit{Are both directions calibrated?} Over $1{,}000$ test pairs, the coverage of every parameter follows the diagonal, and that of every output lies on or slightly above it, meaning slightly conservative intervals where the simulator's spread is small (Figure~\ref{fig:lv}c).

\begin{figure}[t]
\centering
\includegraphics[width=0.75\columnwidth]{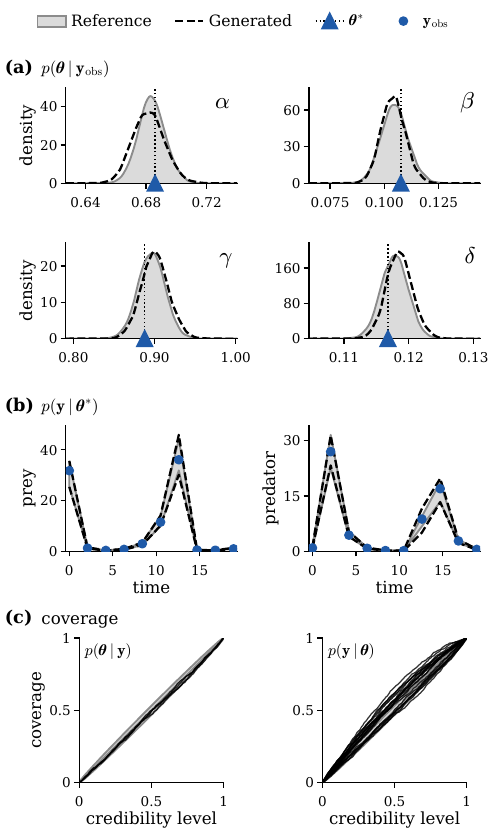}
\caption{\textbf{Lotka--Volterra.} All panels use one trained model. Bands in (b) are central $90\%$ intervals from $10{,}000$ samples each; the gray band in (c) marks the $95\%$ range of a calibrated model.}
\label{fig:lv}
\end{figure}

\subsection{Case study 2: Bayesian Denoising}

Image restoration calls for different inference queries depending on what is known. A blurred, noisy observation may be complete, supplemented by known clean pixels, or available only in part. Each setting corresponds to a different conditional of the same joint distribution (Figure~\ref{fig:image_queries}). We study these queries through Bayesian denoising of Fashion MNIST, recovering a clean image $\bm{\theta} \in \mathbb{R}^{784}$ from its degraded observation $\mathbf{y} \in \mathbb{R}^{784}$~\citep{Radev23}. \textit{OneFlowSBI} stacks the two images and the mask as channels on a shared $28 \times 28$ grid and processes them with a convolutional network (Appendix~\ref{app:hyperparam_image}). Encoding pixel roles as input channels preserves the spatial structure of the data while supporting conditioning on individual pixels. For the $d = 1{,}568$ variables, each convolutional layer has $\mathcal{O}(d)$ cost, whereas Simformer's dense attention with one token per variable costs $\mathcal{O}(d^2)$ per layer.

\textit{Does one model answer the standard queries?} The posterior recovers the shape and texture of the clean image, with uncertainty on the edges and fine patterns that the blur destroys (Figure~\ref{fig:image_queries}, row 1), whereas NPE, trained for this posterior alone, returns a noisier mean with more than twice the error (row 2). Reversed, the same model provides the likelihood, which NPE cannot, and its mean matches the average of $1{,}000$ simulator runs (row 3). Over $100$ test images, \textit{OneFlowSBI} has less than half the RMSE of the best baseline in each of five training runs, and still beats every baseline with an MLP in place of the convolutional network (Table~\ref{tab:denoising_test} in Appendix~\ref{app:add_results_denoising}).

\textit{Does added information sharpen the posterior, and missing information widen it where expected?} Conditioning also on $20\%$ of the clean pixels lowers both the error and the standard deviation (row 4). Marking an $8 \times 8$ block of the observation as absent raises the uncertainty inside and around the block, where the blur spreads the missing information, and leaves it close to row 1 elsewhere (row 5), whereas the uncertainty of Simformer does not respond to the block (Appendix~\ref{app:denoising_block}).

\textit{Are the posterior and likelihood calibrated?} Over $1{,}000$ test pairs, TARP~\citep{tarp}, which checks the full image at once rather than pixel by pixel, shows a posterior close to the diagonal, slightly overconfident, and a conservative likelihood, whose samples vary slightly more than the simulator's (Appendix~\ref{app:denoising_calibration}).

\textit{At what cost?} \textit{OneFlowSBI} trains in $39$ minutes and draws $1{,}000$ posterior samples in $6$ seconds, whereas Simformer needs about $7$ hours to train and $7.5$ minutes to sample with its default sampler, $11$ and $74$ times longer (Table~\ref{tab:case_summary} and Appendix~\ref{app:case_study_cost}). NPE trains in about the same time as \textit{OneFlowSBI} and samples in a single pass, but learns only the posterior. A single \textit{OneFlowSBI} model thus answers every query, with a more accurate posterior than all baselines, at a fraction of the cost of Simformer.

\begin{figure}[t]
  \centering
  \includegraphics[width=\columnwidth]{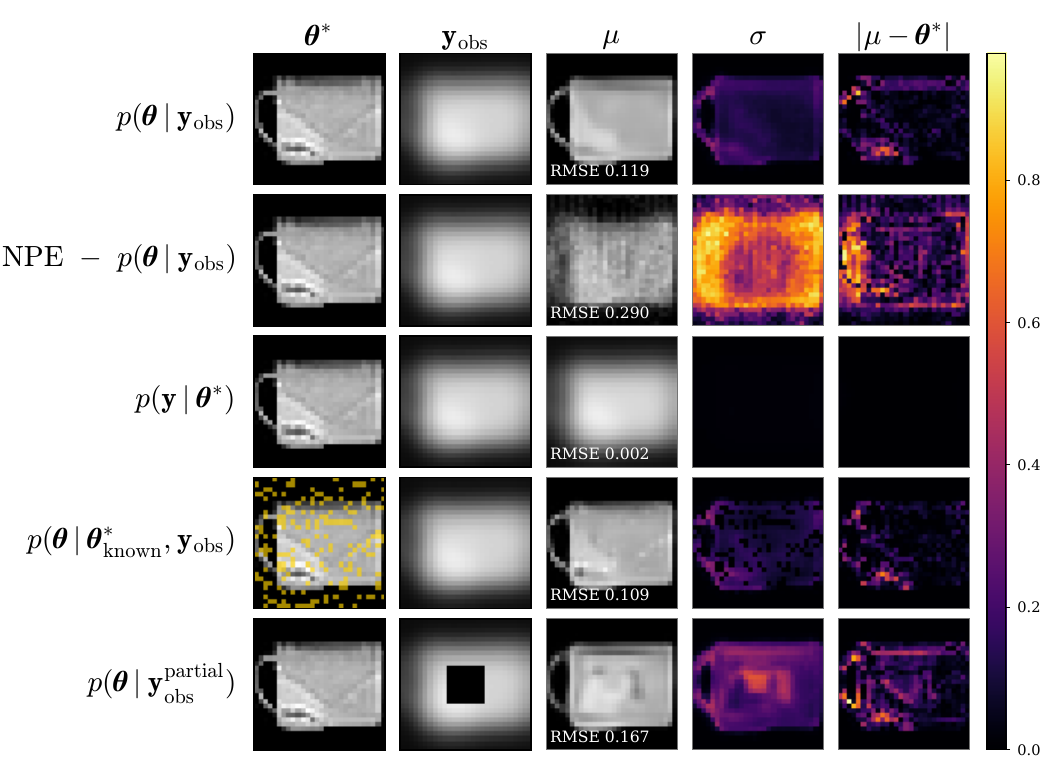}
  \caption{\textbf{Bayesian denoising.} One query per row, all from one trained \textit{OneFlowSBI} model except row 2, which shows the posterior of NPE; revealed clean pixels are in yellow and the missing observation block in black. For the likelihood (row 3), the error is taken against the mean of simulator runs.}
  \label{fig:image_queries}
\end{figure}

\subsection{Case study 3: Shallow water inference}

In monitoring physical systems, data arrive over time, and inference is often needed before the record is complete; each observation window then defines a different conditional of the same joint distribution. We study this on shallow water waves~\citep{Ramesh22}: a disturbance at the left end of a $1{,}000$\,km channel launches a wave whose speed depends on the local depth, and the surface height is recorded at $100$ positions every hour from $0$ to $100$\,h. From these noisy records, $\mathbf{y} \in \mathbb{R}^{2 \times 101 \times 100}$, we infer the depth profile $\bm{\theta} \in \mathbb{R}^{100}$ (Appendix~\ref{app:task_desc_sw}).

\textit{OneFlowSBI} aligns each depth with its position on the grid of hours and positions and processes the joint state with a convolutional network, which exploits the locality of the wave in space and time. With one token per value, a transformer-based generalist would attend over $20{,}300$ tokens, or about $412$ million token pairs in each attention head and layer, too many to train with dense attention on a single GPU; grouping the values of each hour into one token makes this feasible but restricts conditioning to whole hours, whereas \textit{OneFlowSBI} conditions on individual values (Appendix~\ref{app:hyperparam_sw}).

\textit{Can one model provide the posterior as the data arrive?} The data inform the depth of a cell only after the wave has passed over it. For each cutoff $T$ from $5$ to $100$\,h, we mark all later hours as absent and sample the posterior given the first $T$ hours, $p(\bm{\theta} \mid \mathbf{y}_{\mathrm{obs}}^{\le T})$, from the same model, where a specialist posterior estimator would need one model per cutoff. As a check against the physics, we record the arrival time $\tau(x)$ of the wave front at each position $x$ in a noise-free simulation at $\bm{\theta}^{*}$; the model never sees this curve.

The boundary between high and low uncertainty follows $\tau(x)$ across the whole channel (Figure~\ref{fig:sw}a): before the wave arrives, the standard deviation stays near that of the prior, and once the wave has passed, it drops. At $T = 10$\,h, the posterior covers $\bm{\theta}^{*}$ behind the front and stays as wide as the prior ahead of it, and by $T = 100$\,h it recovers the whole profile (Figure~\ref{fig:sw}b). Averaged over $100$ test pairs, the uncertainty behind the front falls steadily with $T$, while ahead of the front it stays far higher, informed only through the prior's correlation between neighbouring depths (Figure~\ref{fig:sw_calib}b).

\textit{Is this uncertainty calibrated?} Over the same test pairs, the coverage of the $50\%$ and $90\%$ intervals is close to their credibility levels at the earliest and latest cutoffs and conservative in between (Figure~\ref{fig:sw_calib}a).

\textit{How does it compare with the baselines, and at what cost?} On the windows, which Simformer trained on all conditionals answers by generating the later hours and discarding them, \textit{OneFlowSBI} is more accurate at every cutoff we compare, and Simformer is overconfident (Appendix~\ref{app:sw_windows}). With the full observation, the one query all methods can answer, \textit{OneFlowSBI} has the most accurate posterior in each of five training runs, with an average error more than a third below that of the best baseline, Simformer trained on the posterior alone, and a quarter of that of NPE (Appendix~\ref{app:add_results_sw}). It also samples about $24$ times faster than Simformer with its default sampler (Table~\ref{tab:case_summary}) and trains faster (Appendix~\ref{app:case_study_cost}).

\begin{table}[!t]
\centering
\small
\setlength{\tabcolsep}{3pt}
\caption{\textbf{Case studies.} RMSE of the posterior mean and sampling time for $1{,}000$ posterior samples, with $30{,}000$ simulations (Appendix~\ref{app:add_results_denoising}--\ref{app:case_study_cost}).}
\label{tab:case_summary}
\begin{tabular}{@{}lcccc@{}}
\toprule
 & \multicolumn{2}{c}{Denoising} & \multicolumn{2}{c}{Shallow water} \\
\cmidrule(lr){2-3}\cmidrule(lr){4-5}
Method & RMSE & Time (s) & RMSE (m) & Time (s) \\
\midrule
\textit{OneFlowSBI} & $\mathbf{0.111}$ & $6.0$ & $\mathbf{0.329}$ & $16.6$ \\
Simformer (all) & $0.326$ & $447$ & $0.795$ & $403$ \\
Simformer (post.) & $0.245$ & $443$ & $0.537$ & $395$ \\
NPE & $0.247$ & $0.08$ & $1.265$ & $0.04$ \\
\bottomrule
\end{tabular}
\end{table}

\begin{figure}[!t]
  \centering
  \includegraphics[width=0.75\columnwidth]{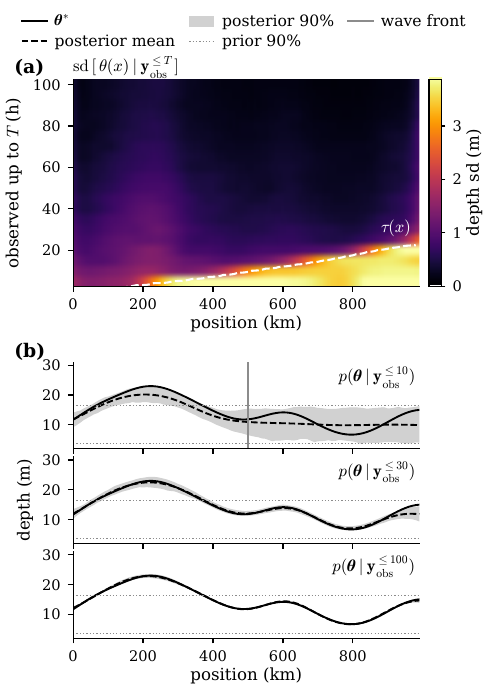}
  \caption{\textbf{Shallow water inference.} All panels use one trained model. (a) Posterior standard deviation at each position and cutoff $T$, with the wave arrival time $\tau(x)$ dashed. (b) Depth posteriors at three cutoffs.}
  \label{fig:sw}

  \vspace{4pt}
  \includegraphics[width=0.75\columnwidth]{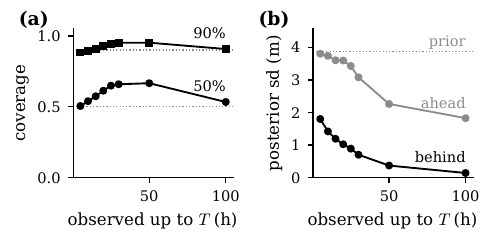}
\caption{\textbf{Shallow water: calibration.} Same model as Figure~\ref{fig:sw}, over $100$ test pairs. (a) Coverage of the $50\%$ and $90\%$ intervals, with their credibility levels dotted. (b) Posterior standard deviation behind and ahead of each pair's wave front; in a few pairs, the wave never reaches some cells.}
  \label{fig:sw_calib}
\end{figure}

\section{Ablation}\label{sec:ablation}
Training a single model for every query raises two natural questions: what is the accuracy cost of this generality, and how much does the result depend on the mixture of training masks? We answer both on Lotka--Volterra, a challenging benchmark task where every query has a reference and case study 1 examines the distributions in depth. Each configuration is trained five times with $30{,}000$ simulations.

\subsection{Amortization gap over queries}\label{sec:generality_ablation}
An amortized method trains one network once and then answers any new observation without retraining; the accuracy it loses relative to fitting each observation separately is the amortization gap~\citep{cremer2018inference}. \textit{OneFlowSBI} amortizes over queries as well, which may add a further gap relative to specialists trained for a single query. On the posterior, the benchmark already shows that, at the largest budget, \textit{OneFlowSBI} is more accurate than NPE, a specialist with a different architecture, on nine of ten tasks; here we isolate the gap with four specialists of the same architecture and training, each trained with a single fixed mask for one query (Table~\ref{tab:ablation}). The posterior is close to its specialist, and the prior and prior predictive of both models are indistinguishable from their references. The gap is concentrated in the likelihood, the most demanding query here: it is $20$-dimensional against $4$ dimensions for the posterior, yet receives only $20\%$ of the training masks, whereas its specialist receives all of them. Even so, the shared likelihood remains far more accurate than Simformer's on this task (Appendix~\ref{app:query_c2st}).

\subsection{Sensitivity to the mask distribution}\label{sec:mask_ablation}
All experiments use one mixture of training masks, fixed across tasks (Appendix~\ref{app:sbi_hyperparams}), which raises the question: \textit{how much does accuracy depend on this choice?} We test three alternatives: no partial masks, equal weights on all families, and half of all masks partial (Figure~\ref{fig:mask_ablation}). Removing partial masks yields the most accurate posterior and likelihood. However, the partial queries on which the case studies rely then have zero training probability, so the population result no longer covers them (Appendix~\ref{app:proof}). Across all four mixtures, the prior and prior predictive remain indistinguishable from their references, and the likelihood is the query most sensitive to the choice of mixture. The mixture thus shifts accuracy between queries and should reflect those a practitioner needs; our default gives every query positive probability.

\begin{table}[H]
\centering
\footnotesize
\setlength{\tabcolsep}{3pt}
\renewcommand{\arraystretch}{1.0}
\caption{\textbf{Cost of generality on Lotka--Volterra.} C2ST (mean $\pm$ SD over five runs), with one specialist per query.}
\label{tab:ablation}
\begin{tabular}{@{}lcc@{}}
\toprule
Query & Specialist & \textit{OneFlowSBI} \\
\midrule
$p(\boldsymbol{\theta}\mid\mathbf{y})$ & $0.697\pm0.013$ & $0.720\pm0.011$ \\
$p(\mathbf{y}\mid\boldsymbol{\theta})$ & $0.604\pm0.013$ & $0.758\pm0.017$ \\
$p(\boldsymbol{\theta})$ & $0.493\pm0.009$ & $0.496\pm0.006$ \\
$p(\mathbf{y})$ & $0.508\pm0.017$ & $0.505\pm0.007$ \\
\midrule
Number of models & 4 & 1 \\
\bottomrule
\end{tabular}
\end{table}

\section{Limitations}
\label{sec:limit}
Sharing one network across queries opens an amortization gap relative to specialists, mostly for the likelihood, and accuracy depends on the mask distribution (Sec.~\ref{sec:ablation}). Because one model learns every distribution from the same simulations, it also appears to require more simulations: at the smallest benchmark budget it is less accurate than the baselines on several tasks, although it gains the most as the budget grows. Observations dominated by uninformative noise, as in SLCP distractors, are harder for a joint model, which must model this noise as well. The freedom to choose the backbone means designing one for each data modality, and the model operates on fixed-dimensional states, so it cannot handle infinite-dimensional parameters or observations and lacks features of transformer-based generalists such as priors specified at inference time or inputs of varying size. Our analysis holds at the population optimum (Appendix~\ref{app:proof}); after finite training, the conditionals need not agree with one joint distribution. 

\begin{figure}[!t]
\centering
\includegraphics[width=0.8\columnwidth]{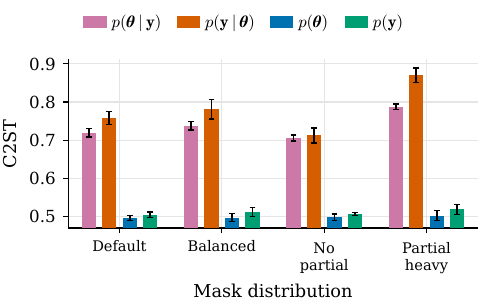}
\caption{\textbf{Sensitivity to the mask distribution} on Lotka--Volterra (mean $\pm$ SD over five runs). Weights on joint, posterior, likelihood, $p(\mathbf{y})$, $p(\bm{\theta})$ and partial masks: default (0.15, 0.2, 0.2, 0.15, 0.1, 0.2), balanced (all 1/6), no partial (0.2, 0.2, 0.2, 0.2, 0.2, 0) and partial heavy (0.1, 0.1, 0.1, 0.1, 0.1, 0.5).}
\label{fig:mask_ablation}
\end{figure}

\section{Conclusion}
\label{sec:conclusion}
\textit{OneFlowSBI} learns the joint distribution of parameters and observations with one masked flow matching network, so posterior, likelihood, marginal and partial queries differ only in the mask. It is not designed to replace specialists, which remain the natural choice when only one query is needed, nor do we claim that it is inherently better than transformer-based generalists; we studied what it offers relative to them and at what cost. Because the mask is an input, the backbone can match the data: the model is competitive with specialist and generalist methods on standard benchmarks and, on high-dimensional problems where token-based generalists become costly, gives a more accurate posterior than all baselines at a fraction of Simformer's cost. In turn, sharing one network across queries leaves an amortization gap relative to specialists, mostly for the likelihood. Future work includes an adaptive mask distribution that samples slowly learned queries more often, and other design choices around flow matching, such as interpolants and couplings that straighten the learned paths and could allow sampling in a few steps.

\subsubsection*{AI use statement}
In this work, we used generative AI tools to support research and manuscript preparation, including feedback on methodology and experiments, assistance with code, discussion of the interpretation of results, and polishing the text to improve readability. The authors reviewed and verified all AI-assisted work, and all research decisions and final interpretations are their own. We take full responsibility for the content of this work.

\bibliography{references}

\clearpage
\appendix
\thispagestyle{empty}

\onecolumn
\aistatstitle{Appendix}

\addtocontents{toc}{\protect\setcounter{tocdepth}{2}}
\tableofcontents
\clearpage

\section{Frequently Asked Questions (FAQ)}
\label{sec:faq}

\paragraph{Q: How does the design of OneFlowSBI differ from transformer-based generalists?}
Transformer-based generalists such as Simformer and ACE specify a query through their tokens: Simformer attaches a condition state to each variable token, and ACE separates context tokens from target tokens. This design handles inputs of varying size naturally, but it ties the cost to the number of tokens: as \citet{gloecklerall} note, transformer evaluations scale quadratically with it, and ACE's attention likewise grows with the number of context and target points~\citep{chang2025amortized}. For high-dimensional data, \citet{gloecklerall} therefore recommend condensing whole vectors into single tokens with embedding networks, which reduces the cost but, as they note, loses direct control over the condition states of individual data elements. 

\textit{OneFlowSBI} instead supplies the query as a mask alongside the joint state, so it does not depend on tokens and can be combined with a backbone suited to the data, such as a convolutional network for images and fields, while keeping conditioning on individual values. Sparse or memory-efficient attention, which \citet{gloecklerall} also suggest, could reduce the memory of token-based models on large inputs, although exact attention still has a cost quadratic in the number of tokens.

\paragraph{Q: What role does flow matching play?}
Supplying the mask as an input does not depend on the generative objective and could equally be combined with diffusion, and our experiments do not separate the objective from the backbone. We use flow matching because it learns the velocity of a deterministic ODE directly: training regresses onto the target $\mathbf{z}_1-\mathbf{z}_0$, whose scale is the same at every time, from an exact Gaussian base, without a noise schedule or time-dependent loss weights, and sampling integrates the ODE with adaptive solvers that control an estimate of the integration error. For this linear path, \citet{lipman2023flow} report faster training and sampling than with diffusion paths. Score-based models instead learn a score and are typically sampled with a stochastic reverse SDE, as in Simformer's default sampler~\citep{gloecklerall}. 

Compared with NPE, which uses a discrete normalizing flow whose layers must be invertible with a tractable Jacobian, flow matching allows unconstrained architectures and scales to large ones~\citep{fmpe}; consistent with this, in our denoising study the MLP version of \textit{OneFlowSBI}, which has no image-specific structure, is still more accurate than NPE while learning every query (Appendix~\ref{app:add_results_denoising}).

\paragraph{Q: What does the absent role add over a binary mask, and when does it apply?}\label{faq:absent}
Write the joint state as $\mathbf{z} = (\mathbf{z}_{\mathrm{O}}, \mathbf{z}_{\mathrm{G}}, \mathbf{z}_{\mathrm{A}})$, where $\mathbf{z}_{\mathrm{A}}$ is what the query leaves out. With a binary mask, $\mathbf{z}_{\mathrm{A}}$ must be generated and then discarded: the samples are valid, but the model generates far more than the query asks for (about $19{,}000$ observation values to obtain $100$ depths for the earliest shallow water window), and the density of the query requires integrating $\mathbf{z}_{\mathrm{A}}$ out. With the absent role, $\mathbf{z}_{\mathrm{A}}$ is replaced by a placeholder that the mask flags and the loss ignores, so the flow learns the target directly,
\begin{equation*}
p(\mathbf{z}_{\mathrm{G}} \mid \mathbf{z}_{\mathrm{O}}) = \int p(\mathbf{z}_{\mathrm{G}}, \mathbf{z}_{\mathrm{A}} \mid \mathbf{z}_{\mathrm{O}})\, d\mathbf{z}_{\mathrm{A}},
\end{equation*}
generating only the requested coordinates, with a density that follows from the flow. In denoising, a missing block of the observation then yields a posterior whose uncertainty concentrates on the missing region (Appendix~\ref{app:denoising_block}). For missing data, this target is appropriate when the pattern of missing values carries no information about the parameters beyond the observed values, as with sensor failures or fixed observation windows; otherwise, the missingness mechanism must be part of the simulator.

\paragraph{Q: Are the learned conditionals consistent with Bayes' rule?}
At the population optimum, every conditional and marginal derives from the same joint distribution, so they satisfy Bayes' rule (Appendix~\ref{app:proof}). After finite training, \textit{OneFlowSBI} does not enforce this. Because each query is sampled by integrating an ODE, its log density can be estimated with the instantaneous change of variables, so self-consistency losses that penalize violations of Bayes' rule~\citep{schmitt2024selfconsistency,mishra2026robust} could be added as auxiliary training terms; we leave this to future work.



\section{Theoretical Foundations}\label{app:theoretical_foundations}

\subsection{Derivation and Interpretation of the Masked Probability Path}\label{sec:masked-path-derivation}

Consider the augmented state
$\mathbf z=\begin{bmatrix}\bm{\theta}\\ \mathbf y\end{bmatrix}$,
with Gaussian base sample
$\mathbf z_0=\begin{bmatrix}\bm{\theta}_0\\ \mathbf y_0\end{bmatrix}
\sim\mathcal N(\mathbf 0,\mathbf I)$
and target sample
$\mathbf z_1=\begin{bmatrix}\bm{\theta}_1\\ \mathbf y_1\end{bmatrix}
\sim p(\bm{\theta})p(\mathbf y\mid\bm{\theta})$.
Different SBI queries correspond to transporting different subsets of the 
augmented state, while the remaining coordinates are either held fixed at their 
conditioning values or excluded from the active probability path.
For coordinates transported from the base distribution to the target
distribution, we use the linear probability path
\[
\mathbf z_t=(1-t)\mathbf z_0+t\mathbf z_1.
\]
We first consider the principal SBI queries and then use their common
structure to derive the general masked probability path.

\paragraph{Joint distribution.}
To learn the joint model $q_\phi(\bm{\theta},\mathbf y)$ approximating
$p(\bm{\theta},\mathbf y)$, both the parameter and observation
blocks are transported along the linear probability path,
\[
\mathbf z_t^{\mathrm{joint}}
=
\begin{bmatrix}
(1-t)\bm{\theta}_0+t\bm{\theta}_1\\
(1-t)\mathbf y_0+t\mathbf y_1
\end{bmatrix}.
\]
Thus, both $\bm{\theta}$ and $\mathbf y$ evolve from their Gaussian base
samples to their corresponding target samples.

\paragraph{Posterior.}
For $q_\phi(\bm{\theta}\mid\mathbf y)$, the parameter block is transported,
while the observation block remains fixed at its conditioning value,
\[
\mathbf z_t^{\mathrm{post}}
=
\begin{bmatrix}
(1-t)\bm{\theta}_0+t\bm{\theta}_1\\
\mathbf y_1
\end{bmatrix}.
\]
Thus, $\bm{\theta}$ is generated along the probability path while
$\mathbf y$ remains fixed throughout the flow.

\paragraph{Likelihood.}
For $q_\phi(\mathbf y\mid\bm{\theta})$, the parameter block remains fixed
at its conditioning value, while the observation block is transported,
\[
\mathbf z_t^{\mathrm{lik}}
=
\begin{bmatrix}
\bm{\theta}_1\\
(1-t)\mathbf y_0+t\mathbf y_1
\end{bmatrix}.
\]
Thus, $\mathbf y$ is generated along the probability path while
$\bm{\theta}$ remains fixed throughout the flow.

\paragraph{Marginal distributions.}
Marginal queries introduce a third behavior: a block may be absent from
the active probability path. For the observation marginal
$q_\phi(\mathbf y)$, the parameter block is absent and only the
observation block is transported,
\[
\mathbf z_t^{p(\mathbf y)}
=
\begin{bmatrix}
\mathbf 0_\theta\\
(1-t)\mathbf y_0+t\mathbf y_1
\end{bmatrix}.
\]
Here, $\mathbf 0_\theta$ does not correspond to conditioning on
$\bm{\theta}=\mathbf 0$. Instead, it represents a parameter block that is
absent from the active probability path. Since the target samples are drawn
from $\mathbf z_1\sim p(\bm{\theta})p(\mathbf y\mid\bm{\theta})$, their
observation components satisfy
$\mathbf y_1\sim p(\mathbf y)
=\int p(\bm{\theta})p(\mathbf y\mid\bm{\theta})\,d\bm{\theta}$.
Consequently, transporting only the $\mathbf y$ block defines a probability
path whose target is the observation marginal $p(\mathbf y)$.

Similarly, for the parameter marginal $q_\phi(\bm{\theta})$, the
observation block is absent and only the parameter block is transported,
\[
\mathbf z_t^{p(\bm{\theta})}
=
\begin{bmatrix}
(1-t)\bm{\theta}_0+t\bm{\theta}_1\\
\mathbf 0_y
\end{bmatrix}.
\]
Here, $\mathbf 0_y$ represents an absent observation block rather than an
observed value of zero. Since $\bm{\theta}_1\sim p(\bm{\theta})$,
transporting only the parameter block defines a probability path whose
target is the parameter marginal $p(\bm{\theta})$.

\paragraph{Motivation for the masking scheme.}
The preceding cases show that every coordinate has one of three possible
roles. We therefore introduce a ternary mask
$\mathbf m\in\{-1,0,1\}^{d}$ and define the role of each mask value directly
through the corresponding coordinate-wise probability path,
\[
z_{t,i}
=
\begin{cases}
0,
& m_i=-1
\qquad \text{(absent)},\\[2mm]
(1-t)z_{0,i}+tz_{1,i},
& m_i=0
\qquad \text{(generated)},\\[2mm]
z_{1,i},
& m_i=1
\qquad \text{(observed)}.
\end{cases}
\]
Thus, $m_i=-1$ excludes coordinate $i$ from the active probability path,
$m_i=0$ transports it from its Gaussian base value to its target value,
and $m_i=1$ keeps it fixed at its conditioning value.

For the augmented state
$\mathbf z=\begin{bmatrix}\bm{\theta}\\ \mathbf y\end{bmatrix}$, the
principal SBI queries therefore correspond to
\[
\begin{aligned}
\mathbf m_{\mathrm{joint}}
&=
\begin{bmatrix}
\mathbf 0_\theta\\
\mathbf 0_y
\end{bmatrix},
&
\mathbf m_{\mathrm{post}}
&=
\begin{bmatrix}
\mathbf 0_\theta\\
\mathbf 1_y
\end{bmatrix},
&
\mathbf m_{\mathrm{lik}}
&=
\begin{bmatrix}
\mathbf 1_\theta\\
\mathbf 0_y
\end{bmatrix},
\\[3mm]
\mathbf m_{p(\mathbf y)}
&=
\begin{bmatrix}
-\mathbf 1_\theta\\
\mathbf 0_y
\end{bmatrix},
&
\mathbf m_{p(\bm{\theta})}
&=
\begin{bmatrix}
\mathbf 0_\theta\\
-\mathbf 1_y
\end{bmatrix}.
\end{aligned}
\]

To express the three coordinate-wise behaviors in a common vector form,
we decompose the ternary mask $\mathbf m\in\{-1,0,1\}^{d}$ into two
binary masks $\mathbf m^{\mathrm{obs}},\mathbf m^{\mathrm{gen}}
\in\{0,1\}^{d}$. The observed mask $\mathbf m^{\mathrm{obs}}$ identifies
coordinates that are held fixed at their conditioning values, while the
generated mask $\mathbf m^{\mathrm{gen}}$ identifies coordinates that are
transported along the probability path. For each coordinate $i$,

\[
m_i^{\mathrm{obs}}
=
\begin{cases}
1, & m_i=1,\\
0, & \text{otherwise},
\end{cases}
\qquad
m_i^{\mathrm{gen}}
=
\begin{cases}
1, & m_i=0,\\
0, & \text{otherwise}.
\end{cases}
\]
Thus, an observed coordinate is selected only by
$\mathbf m^{\mathrm{obs}}$, a generated coordinate is selected only by
$\mathbf m^{\mathrm{gen}}$, and an absent coordinate is selected by
neither. Equivalently,
\[
m_i=-1
\Rightarrow
(m_i^{\mathrm{obs}},m_i^{\mathrm{gen}})=(0,0),
\qquad
m_i=0
\Rightarrow
(0,1),
\qquad
m_i=1
\Rightarrow
(1,0).
\]

The coordinate-wise probability paths can therefore be written compactly
as
\begin{equation}
\mathbf z_t
=
\mathbf m^{\mathrm{obs}}\odot\mathbf z_1
+
\mathbf m^{\mathrm{gen}}\odot
\left[(1-t)\mathbf z_0+t\mathbf z_1\right].
\label{eq:masked-interpolant_supp}
\end{equation}
Here, $\mathbf m^{\mathrm{obs}}\odot\mathbf z_1$ keeps the observed
coordinates fixed at their conditioning values, while
$\mathbf m^{\mathrm{gen}}\odot[(1-t)\mathbf z_0+t\mathbf z_1]$
transports the generated coordinates from Gaussian noise to their
corresponding target values. Coordinates marked as absent have zero entries
in both binary masks and therefore contribute to neither term.
Consequently, Eq.~\eqref{eq:masked-interpolant_supp} recovers all of the special
cases above.

For the linear probability path, differentiating
Eq.~\eqref{eq:masked-interpolant_supp} with respect to $t$ gives the corresponding
conditional target velocity
\begin{equation}
\mathbf{v}_t(\mathbf z_t\mid\mathbf z_1, \mathbf{m})
=
\mathbf m^{\mathrm{gen}}\odot
(\mathbf z_1-\mathbf z_0).
\label{eq:masked-target-velocity}
\end{equation}
Thus, the target velocity is zero on observed and absent coordinates.

\paragraph{Training with absent coordinates.}
Absent coordinates do not contribute directly to the flow matching loss,
since no velocity needs to be learned for coordinates outside the active
probability path. However, the ternary mask $\mathbf m$ is provided as an
input to the velocity field. Consequently, training examples containing
$m_i=-1$ still supervise the velocity predictions of the generated
coordinates under that mask configuration. Because absent coordinates are
represented by fixed placeholders in $\mathbf z_t$, their target values are
not provided to the velocity field. For example, under the marginal mask
$\mathbf m=[-\mathbf 1_\theta;\mathbf 0_y]$, no velocity is learned for the
absent $\bm{\theta}$ coordinates, while the generated $\mathbf y$
coordinates are trained to follow their marginal probability path. Thus,
$-1$ acts as conditioning information indicating that a coordinate is
absent, rather than as a coordinate whose velocity must itself be predicted.

\paragraph{Mixed conditionals.}
The same formulation naturally extends to arbitrary mixed conditional
queries. Suppose the parameter and observation vectors are partitioned as
$\bm{\theta}
=\begin{bmatrix}\bm{\theta}_{\mathrm{gen}}\\
\bm{\theta}_{\mathrm{obs}}\end{bmatrix}$
and
$\mathbf y
=\begin{bmatrix}\mathbf y_{\mathrm{gen}}\\
\mathbf y_{\mathrm{obs}}\end{bmatrix}$.
When all coordinates participate in the modeled conditional distribution,
the corresponding mask is
\[
\mathbf m_{\mathrm{partial}}
=
\begin{bmatrix}
\mathbf 0_{\theta_{\mathrm{gen}}}\\
\mathbf 1_{\theta_{\mathrm{obs}}}\\
\mathbf 0_{y_{\mathrm{gen}}}\\
\mathbf 1_{y_{\mathrm{obs}}}
\end{bmatrix}.
\]
Substituting this mask into Eq.~\eqref{eq:masked-interpolant_supp} gives
\[
\mathbf z_t
=
\begin{bmatrix}
(1-t)\bm{\theta}_{{\rm gen},0}
+t\bm{\theta}_{{\rm gen},1}\\
\bm{\theta}_{{\rm obs},1}\\
(1-t)\mathbf y_{{\rm gen},0}
+t\mathbf y_{{\rm gen},1}\\
\mathbf y_{{\rm obs},1}
\end{bmatrix}.
\]
Hence arbitrary subsets of both $\bm{\theta}$ and $\mathbf y$ can be
generated while the complementary subsets remain fixed at their
conditioning values.

\paragraph{Missing observations.}
The ternary formulation also permits all three coordinate roles to appear
within the same query. If $G$, $O$, and $A$ denote the generated, observed,
and absent coordinate sets, respectively, the corresponding probability
path models the generated coordinates conditional on the observed
coordinates, with the absent coordinates marginalized out,
\[
p(\mathbf z_G\mid\mathbf z_O)
=
\int
p(\mathbf z_G,\mathbf z_A\mid\mathbf z_O)\,
d\mathbf z_A.
\]

This provides a natural treatment of partially observed data. Let
$\mathbf y=
\begin{bmatrix}
\mathbf y_{\mathrm{obs}}\\
\mathbf y_{\mathrm{miss}}
\end{bmatrix}$.
For posterior inference, $\bm{\theta}$ is treated as generated,
$\mathbf y_{\mathrm{obs}}$ as observed, and
$\mathbf y_{\mathrm{miss}}$ as absent. The corresponding target
distribution is therefore
\[
p(\bm{\theta}\mid\mathbf y_{\mathrm{obs}})
=
\int
p(
\bm{\theta},
\mathbf y_{\mathrm{miss}}
\mid
\mathbf y_{\mathrm{obs}})
\,d\mathbf y_{\mathrm{miss}}.
\]
The absent-mask query consequently trains
$q_\phi(\bm{\theta}\mid\mathbf y_{\mathrm{obs}})$ directly toward this
marginalized target, without requiring the missing observations to be
explicitly generated.

\subsection{Validity of Masked Conditional and Marginal Probability Flows}
\label{app:proof}

For a fixed ternary mask $\mathbf m$, let $G$, $O$, and $A$ denote the
generated, observed, and absent coordinate sets. The target is
\begin{equation}
p(\mathbf z_G\mid\mathbf z_O)
=
\int p(\mathbf z_G,\mathbf z_A\mid\mathbf z_O)\,d\mathbf z_A.
\label{eq:ternary-query-target}
\end{equation}
Assume that the Gaussian base sample $\mathbf z_0$, target sample
$\mathbf z_1$, mask $\mathbf m$, and time $t\sim\mathcal U(0,1)$ are
sampled independently, with the base and target samples having finite second moments. 
Applying standard flow matching arguments~\citep{lipman2023flow}
to the generated subspace then yields the following population
validity result.

\begin{proposition}[Population validity of masked flow matching]
\label{prop:masked-fm-validity}
For each mask with positive training probability, the unrestricted
population minimizer of the masked flow matching objective satisfies
\begin{equation}
\mathbf m^{\mathrm{gen}}\odot
\mathbf v_t^*(\mathbf z_t,\mathbf m)
=
\mathbb E\!\left[
\mathbf m^{\mathrm{gen}}\odot(\mathbf z_1-\mathbf z_0)
\mid\mathbf z_t,\mathbf m
\right],
\label{eq:masked-population-velocity}
\end{equation}
for almost every $t\in(0,1)$. The generated component of this field
satisfies the weak conditional continuity equation for the masked
interpolation path, for almost every conditioning value under
$p(\mathbf z_O)$.
\end{proposition}

\begin{proof}
Squared-error regression is minimized by the conditional expectation;
the fixed-mask factor $1/n_{\mathrm{gen}}$ does not change this
minimizer. For any smooth compactly supported test function $h$ on the
generated subspace,
\begin{align}
\frac{d}{dt}
\mathbb E\!\left[
h(\mathbf z_{t,G})\mid\mathbf z_{1,O},\mathbf m
\right]
&=
\mathbb E\!\left[
\nabla h(\mathbf z_{t,G})^\top
(\mathbf z_{1,G}-\mathbf z_{0,G})
\mid\mathbf z_{1,O},\mathbf m
\right]
\nonumber\\
&=
\mathbb E\!\left[
\nabla h(\mathbf z_{t,G})^\top
[\mathbf v_t^*(\mathbf z_t,\mathbf m)]_G
\mid\mathbf z_{1,O},\mathbf m
\right].
\label{eq:masked-weak-continuity}
\end{align}
This is the weak conditional continuity equation. The interpolant starts
from independent Gaussian noise on the generated coordinates and ends at
$\mathbf z_{1,G}$, whose conditional distribution is
$p(\mathbf z_G\mid\mathbf z_O)$, with absent coordinates marginalized out.
\end{proof}

Under regularity conditions ensuring a well-posed ODE flow and uniqueness
of the corresponding continuity-equation solution on $[0,1]$,
integrating the projected population field realizes this target. The
result concerns the unrestricted population optimum for trained masks;
finite training does not guarantee exact target recovery or cross-mask
Bayes coherence.

\subsection{Sampling from Different SBI Distributions}
\label{sec:sampling-sbi-distributions}

The sampling procedure in the main text specializes to different SBI queries through the choice of ternary mask. Generated coordinates are initialized from Gaussian noise, observed coordinates are initialized at their conditioning values, and absent coordinates use fixed placeholders. The learned ODE is then integrated from $t=0$ to $t=1$, with only generated coordinates evolving. We write $\boldsymbol\varepsilon_\theta\sim\mathcal N(\mathbf0,\mathbf I_{d_\theta})$ and $\boldsymbol\varepsilon_y\sim\mathcal N(\mathbf0,\mathbf I_{d_y})$.

\begin{table}[H]
\centering
\caption{\textbf{Sampling under canonical query masks.}}
\label{tab:sampling-query-masks}
\small
\begin{tabular}{l l l l}
\toprule
\textbf{Query} & \textbf{Mask} & \textbf{Initial state} & \textbf{Output} \\
\midrule
Joint & $[\mathbf0_\theta;\mathbf0_y]$ & $[\boldsymbol\varepsilon_\theta;\boldsymbol\varepsilon_y]$ & $(\bm\theta_1,\mathbf y_1)$ \\
Posterior & $[\mathbf0_\theta;\mathbf1_y]$ & $[\boldsymbol\varepsilon_\theta;\mathbf y_{\rm obs}]$ & $\bm\theta_1$ \\
Likelihood & $[\mathbf1_\theta;\mathbf0_y]$ & $[\bm\theta_{\rm obs};\boldsymbol\varepsilon_y]$ & $\mathbf y_1$ \\
Marginal $p(\mathbf y)$ & $[-\mathbf1_\theta;\mathbf0_y]$ & $[\mathbf0_\theta;\boldsymbol\varepsilon_y]$ & $\mathbf y_1$ \\
Marginal $p(\bm\theta)$ & $[\mathbf0_\theta;-\mathbf1_y]$ & $[\boldsymbol\varepsilon_\theta;\mathbf0_y]$ & $\bm\theta_1$ \\
\bottomrule
\end{tabular}
\end{table}

\paragraph{Joint sampling.}
For $\mathbf m_{\rm joint}=[\mathbf0_\theta;\mathbf0_y]$, both blocks are generated and $\mathbf z_0=[\boldsymbol\varepsilon_\theta;\boldsymbol\varepsilon_y]$. The trajectory and dynamics are
\begin{equation}
\mathbf z_t=
\begin{bmatrix}\bm\theta_t\\ \mathbf y_t\end{bmatrix},
\qquad
\frac{d\mathbf z_t}{dt}=
\begin{bmatrix}
[\mathbf v_t^\phi(\mathbf z_t,\mathbf m_{\rm joint})]_\theta\\
[\mathbf v_t^\phi(\mathbf z_t,\mathbf m_{\rm joint})]_y
\end{bmatrix}.
\end{equation}
The terminal state $(\bm\theta_1,\mathbf y_1)$ is a joint sample.

\paragraph{Posterior sampling.}
For $\mathbf m_{\rm post}=[\mathbf0_\theta;\mathbf1_y]$, we initialize $\mathbf z_0=[\boldsymbol\varepsilon_\theta;\mathbf y_{\rm obs}]$. The observation remains fixed while the parameter block evolves:
\begin{equation}
\mathbf z_t=
\begin{bmatrix}\bm\theta_t\\ \mathbf y_{\rm obs}\end{bmatrix},
\qquad
\frac{d\mathbf z_t}{dt}=
\begin{bmatrix}
[\mathbf v_t^\phi(\mathbf z_t,\mathbf m_{\rm post})]_\theta\\
\mathbf0_y
\end{bmatrix}.
\end{equation}
The terminal parameter block $\bm\theta_1$ is a sample from $q_\phi(\bm\theta\mid\mathbf y_{\rm obs})$.

\paragraph{Likelihood sampling.}
For $\mathbf m_{\rm like}=[\mathbf1_\theta;\mathbf0_y]$, we initialize $\mathbf z_0=[\bm\theta_{\rm obs};\boldsymbol\varepsilon_y]$. The parameter block remains fixed while the observation block evolves:
\begin{equation}
\mathbf z_t=
\begin{bmatrix}\bm\theta_{\rm obs}\\ \mathbf y_t\end{bmatrix},
\qquad
\frac{d\mathbf z_t}{dt}=
\begin{bmatrix}
\mathbf0_\theta\\
[\mathbf v_t^\phi(\mathbf z_t,\mathbf m_{\rm like})]_y
\end{bmatrix}.
\end{equation}
The terminal observation block $\mathbf y_1$ is a sample from $q_\phi(\mathbf y\mid\bm\theta_{\rm obs})$.

\paragraph{Observation-marginal sampling.}
For $\mathbf m_{\rm marg-y}=[-\mathbf1_\theta;\mathbf0_y]$, we initialize $\mathbf z_0=[\mathbf0_\theta;\boldsymbol\varepsilon_y]$. The parameter block is absent and only the observation block evolves:
\begin{equation}
\mathbf z_t=
\begin{bmatrix}\mathbf0_\theta\\ \mathbf y_t\end{bmatrix},
\qquad
\frac{d\mathbf z_t}{dt}=
\begin{bmatrix}
\mathbf0_\theta\\
[\mathbf v_t^\phi(\mathbf z_t,\mathbf m_{\rm marg-y})]_y
\end{bmatrix}.
\end{equation}
The terminal block $\mathbf y_1$ is a sample from $q_\phi(\mathbf y)$.

\paragraph{Parameter-marginal sampling.}
For $\mathbf m_{\rm marg-\theta}=[\mathbf0_\theta;-\mathbf1_y]$, we initialize $\mathbf z_0=[\boldsymbol\varepsilon_\theta;\mathbf0_y]$. The observation block is absent and only the parameter block evolves:
\begin{equation}
\mathbf z_t=
\begin{bmatrix}\bm\theta_t\\ \mathbf0_y\end{bmatrix},
\qquad
\frac{d\mathbf z_t}{dt}=
\begin{bmatrix}
[\mathbf v_t^\phi(\mathbf z_t,\mathbf m_{\rm marg-\theta})]_\theta\\
\mathbf0_y
\end{bmatrix}.
\end{equation}
The terminal block $\bm\theta_1$ is a sample from $q_\phi(\bm\theta)$.

\paragraph{General queries.}
More generally, let $G$, $O$, and $A$ denote generated, observed, and
absent coordinate sets, respectively. Then
\begin{equation}
\mathbf z_t=
\begin{bmatrix}
\mathbf z_{t,G}\\
\mathbf z_{\mathrm{obs},O}\\
\mathbf0_A
\end{bmatrix},
\qquad
\frac{d\mathbf z_t}{dt}=
\begin{bmatrix}
[\mathbf v_t^\phi(\mathbf z_t,\mathbf m)]_G\\
\mathbf0_O\\
\mathbf0_A
\end{bmatrix}.
\end{equation}
Thus, the same learned vector field can be queried with different masks
to obtain joint, conditional, marginal, and mixed-conditional samples.

\newpage

\section{Task Descriptions}\label{app:task_desc}
In this section, we describe the priors defining the parameter inference search space, and simulator settings used to generate the datasets for all benchmark problems reported in the main text.

\subsection{SBIBM Tasks}\label{app:task_sbibm}
The Simulation-Based Inference Benchmark (SBIBM) suite comprises ten canonical problems designed to evaluate inference methods across a range of posterior structures and simulator behaviors encountered in practice. The tasks span linear and nonlinear models, stochastic dynamics, and multimodal posteriors, including Gaussian linear models, SLCP variants, Bernoulli GLMs, Gaussian mixtures, two-moons, SIR, and the Lotka-Volterra predator-prey model. We adopt the standard priors, simulators, and protocols; full specifications are in \citet{lueckmann2021benchmarking} (supplementary Tasks T.1--T.10).

\subsection{Bayesian Denoising}\label{app:task_desc_image}
We use the Bayesian denoising task and implementation of \citet{Radev23}. The prior is implicit: $\bm{\theta} \in \mathbb{R}^{784}$ is a clean $28 \times 28$ Fashion MNIST image drawn uniformly at random, without its class label, so only prior samples are needed, not a prior density. The simulator corrupts the image with Poisson noise and then blurs it,
\begin{equation}
\mathbf{n} = \tfrac{1}{V}\,\mathrm{Poisson}(V\bm{\theta}), \qquad \mathbf{y} = \mathcal{B}_{\sigma}(\mathbf{n}),
\end{equation}
where $\mathbf{n}$ is clipped to $[0, 1]$, $V$ is the number of distinct gray levels in the image rounded up to a power of two (at most $256$), and $\mathcal{B}_{\sigma}$ is a convolution with an isotropic Gaussian kernel of standard deviation $\sigma = 2.5$. The blur removes fine detail and correlates the signal-dependent noise across neighbouring pixels, so the clean image is not uniquely determined by the observation. Both images are mapped to $[-1, 1]$ and dequantized with Gaussian noise of standard deviation $10^{-3}$.

\subsection{Shallow Water Inference}\label{app:task_desc_sw}
The goal is to infer the depth profile of a $1{,}000$\,km channel divided into $100$ cells of $10$\,km, $\bm{\theta} \in \mathbb{R}^{100}$, with the Gaussian prior of \citet{Ramesh22},
\begin{equation}
p(\bm{\theta}) = \mathcal{N}(\mu \mathbf{1}_{100}, \bm{\Sigma}), \qquad
\Sigma_{ij} = \sigma^{2} \exp\!\Big( -\frac{(i-j)^{2}}{2\ell} \Big),
\end{equation}
where $\mu = 10$\,m, $\sigma^{2} = 15$\,m$^{2}$ and $\ell = 100$, which gives smooth profiles.

We use the simulator of \citet{Ramesh22}, which solves the one dimensional Saint Venant equations: a surface disturbance at the left end launches a wave whose speed depends on the local depth, and the two end cells are dry and act as walls. The surface height is recorded every hour for $100$\,h, giving $\bm{\eta}(\bm{\theta}) \in \mathbb{R}^{101 \times 100}$, and the observation is its two dimensional Fourier transform with additive noise,
\begin{equation}
\mathbf{y} = \mathcal{F}\big(\bm{\eta}(\bm{\theta})\big) + \sigma_{\varepsilon}\,\bm{\varepsilon}, \qquad \varepsilon_{kij} \sim \mathcal{N}(0, 1),
\end{equation}
with real and imaginary parts stacked, $\mathbf{y} \in \mathbb{R}^{2 \times 101 \times 100}$, and $\sigma_{\varepsilon} = 0.25$.

So that masks can select individual hours and positions, the model operates on the wavefield in space and time rather than on its Fourier transform. Since $\mathcal{F}$ is invertible, we apply its exact inverse to $\mathbf{y}$: the real part of $\mathcal{F}^{-1}(\mathbf{y})$ is $\bm{\eta}(\bm{\theta})$ plus white noise with standard deviation $\sigma_{\varepsilon}/\sqrt{101 \cdot 100} \approx 0.0025$\,m, and the imaginary part is pure noise of the same scale. We keep both parts, so no information is lost.

\newpage
\section{Implementation Details}
\label{app:implementation}

\subsection{OneFlowSBI on SBIBM}\label{app:sbi_hyperparams}
We evaluate \textit{OneFlowSBI} on all ten SBIBM tasks with simulation budgets of
$10{,}000$, $20{,}000$, and $30{,}000$. We use one general configuration intended
to work well across most tasks, without tuning the architecture or training
hyperparameters separately for each task (except for the number of training steps on SLCP distractors, described below).

\paragraph{Compute environment.} All experiments were run on a supercomputing cluster, with each model trained and evaluated on a single NVIDIA GH200 GPU with 96\,GB of memory and an ARM64 Grace CPU. We use Python 3.12.14; \textit{OneFlowSBI}, NPE and ACE are implemented in PyTorch, and Simformer in JAX.

\paragraph{Architecture.}
We parameterize the vector field using a residual network
with four blocks, hidden width 256 and SiLU activations. The input concatenates the joint state with
one-hot encodings of the mask roles (absent, generated,
or observed). A 256-dimensional sinusoidal time embedding
is processed by an MLP to produce block-specific scale
and shift parameters for adaptive layer normalization \citep{film}.
A final layer normalization and linear projection produce
the velocity predictions. The residual output layers and
final projection are initialized to zero.

\paragraph{Training.}
We train for 50,000 steps with batch size 512 using AdamW \citep{loshchilov2018decoupled}
($\beta_1=0.9$, $\beta_2=0.999$, weight decay $10^{-4}$).
The one exception is SLCP with distractors, which we train for 15,000 steps its posterior overfits the noise coordinates with longer training. The learning rate increases linearly to $5\times10^{-4}$
over 2,000 steps, then follows cosine decay to $10^{-5}$ at the
final step. We clip gradients to a global norm of 10 and use an
exponential moving average with decay 0.999 for evaluation.
Flow matching times are sampled uniformly from $[0,1]$,
with a standard Gaussian base distribution.

\paragraph{Sampling.}
We draw 10,000 samples per query with the adaptive Dormand--Prince solver~\citep{dormand1980family} (\texttt{dopri5} in \texttt{torchdiffeq}~\citep{torchdiffeq}), using relative tolerance $10^{-5}$ and absolute tolerance $10^{-6}$. The solver chooses its step sizes to keep an estimate of the local integration error within the set tolerances, so the number of steps needs no tuning. Because sampling remains fast even at these tight tolerances, taking one to two seconds per query (Appendix~\ref{app:train_sample_times}), we use the same solver settings for every task and query.

\paragraph{Preprocessing.}
We standardize each coordinate of $\bm{\theta}$ and $\mathbf{y}$ with its own mean and standard deviation; four tasks need additional steps. For Lotka--Volterra, observations are log-transformed, $\log(y_j + 10^{-4})$, before standardization. The discrete observations of SIR and Bernoulli GLM (raw) are dequantized with uniform noise: for SIR, fresh $\mathcal{U}(0, 10^{-3})$ noise is added to the generated observation coordinates only; for Bernoulli GLM (raw), $\mathcal{U}(0, 1)$ noise is added to all observation coordinates, including conditioning ones, and redrawn in every batch and for every posterior sample. For reporting, generated observations are discretized as $\operatorname{clip}(q\lfloor \widetilde{y}_j/q \rfloor, 0, 1)$, with $q = 10^{-3}$ for SIR and $q = 1$ for Bernoulli GLM (raw); likelihood calibration instead uses the continuous values to avoid ties. For SLCP with distractors, $92$ of the $100$ observation coordinates are Student-$t$ noise with two degrees of freedom, whose infinite variance violates the finite second moments assumed by the flow matching objective; we therefore apply $\operatorname{asinh}((y_j - c_j)/s_j)$ before standardization, with $c_j$ and $s_j$ the training median and median absolute deviation of coordinate $j$, and invert this map for generated observations.

\paragraph{Training-Mask Settings.}
We use the same mask distribution across all SBI tasks and
simulation budgets: 15\% joint, 20\% posterior, 20\% likelihood,
15\% observation-marginal, 10\% parameter-marginal, and 20\%
partially observed queries. The canonical masks are defined
in Table~\ref{tab:masking_tasks}. A fresh ternary mask is sampled
independently for each training example at every step.
For partially observed queries, separate role-probability
vectors for the parameter and observation blocks are drawn
from $\operatorname{Dirichlet}(1,1,1)$, and each coordinate
is assigned as absent, generated, or observed.
If no coordinate is generated, one is selected uniformly
and assigned the generated role.

\subsection{SBIBM Baselines}

\paragraph{NPE}
We configure NPE following \citet{lueckmann2021benchmarking}, whose
evaluation on the SBIBM tasks found neural spline flows to be consistently
more accurate than the masked autoregressive flows originally proposed for
NPE. We therefore use a neural spline flow as the density
estimator in every experiment, with the implementation in the \textit{sbi}
package \citep{tejero-cantero2020sbi}.

\paragraph{Simformer}
We use Simformer's default configuration for the SBI
benchmark~\citep{gloecklerall}: a transformer with six layers and four
attention heads, in which every scalar coordinate of $\bm{\theta}$ and
$\mathbf{y}$ is a token and attention is dense. It is trained on all
conditionals by denoising score matching under a variance-exploding SDE, with the authors' structured random masks: each training example uses the joint, posterior or likelihood mask, or a random mask that conditions each coordinate with probability 0.3 or 0.7, each of these five cases with
probability 0.2. Training uses Adam with batch size 1,000 and a learning
rate decaying from $10^{-3}$ to $10^{-6}$. Following the authors'
schedule, the number of training steps is three times the number of
simulations, so training length grows with the budget: $30{,}000$,
$60{,}000$ and $90{,}000$ steps for the three budgets. Samples are drawn
with 500 Euler--Maruyama steps, and marginals are obtained by sampling the
joint and discarding the other block.

\paragraph{ACE}
ACE's authors do not evaluate it on the SBIBM suite, so we use the configuration of their SBI experiments~\citep[App.~C.4]{chang2025amortized} for every task and budget. Each parameter and observation coordinate is a token with a $64$-dimensional embedding; the transformer has $6$ layers, $4$ attention heads and feed-forward width $128$, and the output head predicts a mixture of $20$ Gaussians per target coordinate. Training uses Adam with learning rate $5 \times 10^{-4}$, cosine annealed to zero over $50{,}000$ updates with batch size $32$. Each training example places $N_d \sim \mathcal{U}(0, d_y)$ observation tokens and, with probability $0.5$, a random subset of parameter tokens in the context, which is never empty; all other tokens are targets~\citep[App.~C.4.2]{chang2025amortized}. We follow this description rather than the released training sampler, which never places the full observation in the context together with parameters, a combination our queries require. Because ACE predicts each target coordinate independently, we draw multivariate samples autoregressively, in a random order per sample~\citep[App.~B.4]{chang2025amortized}. The prior, prior predictive and joint lie outside this training protocol, so we do not report them for ACE.

\subsection{Case Studies}\label{app:hyperparam_case}
Both case studies share the following settings; the subsections below give only what differs. Since Simformer and ACE are both transformers with one token per variable, we compare against one of them: Simformer, which answers every query type, including the joint and marginals, whereas these lie outside the training protocol of ACE (Section~\ref{sec:related}).

\paragraph{OneFlowSBI} Both case studies use the same mask distribution as the SBIBM tasks and share one training protocol: $30{,}000$ simulations, $50{,}000$ iterations with batch size $512$, Adam~\citep{adam} with a constant learning rate of $5 \times 10^{-4}$, gradient clipping at $1.0$, and an exponential moving average with decay $0.999$; we evaluate the final checkpoint. In both, the role of every value (observed, generated or absent) enters the network as a one-hot input channel, so one network handles all query types, and each residual block applies group normalization, SiLU activations and two $3 \times 3$ convolutions, with a scale and shift set by a sinusoidal time embedding~\citep{film}. We train five models per case study with different random seeds.

\paragraph{NPE} NPE uses a neural spline flow with $10$ rational quadratic coupling transforms and $256$ hidden units. Training uses batch size $512$ and learning rate $5 \times 10^{-4}$ for at most $1{,}000$ epochs, with $10\%$ of the simulations held out; training stops once the validation loss has not improved for $20$ epochs, and the network with the lowest validation loss is kept.

\paragraph{Simformer} Simformer~\citep{gloecklerall} was not evaluated on either problem by its authors, so we train two configurations that share the architecture and training and differ only in the condition mask: the structured random mask used on SBIBM, which learns all conditionals in one model, and a posterior-only mask, which learns $p(\bm{\theta}\mid\mathbf{y})$ alone. The posterior-only model checks that the architecture and training can solve the problem, so differences between the two reflect the cost of learning all conditionals. We use the authors' default large configuration, a transformer with $8$ layers and $4$ attention heads trained by denoising score matching under a variance-exploding SDE, for $100{,}000$ steps, the upper limit of the authors' schedule, with learning rate $5\times10^{-4}$ annealed linearly to $10^{-6}$ over the second half. Samples are drawn with $500$ Euler--Maruyama steps.

\subsubsection{Bayesian Denoising}\label{app:hyperparam_image}

\paragraph{OneFlowSBI} The parameters and observations are represented as $28 \times 28$ images, giving a joint state with two channels; with a three-channel role map for each image, the input has eight channels. A $3 \times 3$ convolution projects the input to $128$ feature channels, followed by four residual blocks, each with group normalization, SiLU activations and two $3 \times 3$ convolutions that preserve spatial resolution. A final $3 \times 3$ convolution, initialized to zero, produces the velocities of the parameter and observation images.

\paragraph{NPE} NPE is conditioned on the flattened observation rather than on a learned summary, since pooling over the image would discard the spatial information each pixel's posterior depends on.  A diagonal Gaussian fitted to the training images serves as the prior; it only defines the posterior object, since single-round NPE trains on simulations drawn from the prior and never evaluates its density. Training stopped after $291$ epochs.

\paragraph{Simformer} Every pixel of $\bm{\theta}$ and $\mathbf{y}$ is a separate token, giving $1{,}568$ tokens with dense attention, and training uses batch size $128$, the largest that fits in GPU memory.

\subsubsection{Shallow Water Inference}\label{app:hyperparam_sw}

\paragraph{OneFlowSBI} The joint state is laid out on the $101 \times 100$ grid of hours and positions: the two observation channels occupy the grid directly, and the depth profile is aligned with the position axis and repeated over time. With the role encoding and two coordinate channels, this gives $14$ input channels. Parameters and observations are standardized with statistics of the training set. The grid is padded to $104 \times 100$ and embedded by a $4 \times 4$ convolution with stride $4$ into $128$ channels, followed by six residual blocks whose convolutions use dilations cycling through $1$, $2$, $4$ and $8$. A transposed convolution returns the features to full resolution, where they are concatenated with the input grid, and two convolutions produce the velocities of the observations and of $\bm{\theta}$, the latter averaged over time.

\paragraph{NPE} Since the observation has $20{,}200$ values, a convolutional network compresses it into a learned summary of dimension $128$, on which the flow is conditioned. The network has five convolutions with kernel size $4$, stride $1$ for the first and $2$ for the others, and $256$, $128$, $64$, $64$ and $16$ output channels, each with spectral normalization and a ReLU activation and all but the last with batch normalization; its output is pooled to a $6 \times 6$ grid and mapped linearly to the summary. The prior is the Gaussian prior of the task, transformed by the same standardization as the parameters. Training stopped after $127$ epochs.

\paragraph{Simformer} One token per value would give $20{,}300$ tokens. With the dense attention of the authors' implementation, storing the attention weights of all $8$ layers and $4$ heads for backpropagation takes about $53$\,GB per training example in single precision, so a single GH200 GPU cannot hold a useful batch; memory-efficient attention would avoid storing these weights, but not the cost of computing them, which grows quadratically with the number of tokens. Following Simformer's use of embedding networks for high-dimensional data~\citep{gloecklerall}, we keep one token per depth and map the $200$ values of each hour to one token with a one-dimensional convolutional encoder, giving $201$ tokens; unlike a single token for the whole observation, this keeps the window queries expressible. Since each token carries one condition state, we sample the condition masks over these tokens rather than over individual values, so each hour is conditioned on or generated as a whole. We also divide each sample's loss by its number of generated values, as in \textit{OneFlowSBI}: without this, samples that generate the $20{,}200$ observation values dominate the gradient over posterior samples, which generate only the $100$ depths, and the all-conditionals model returns the prior; for the posterior-only mask, the division is by a constant. Training uses batch size $512$.

\newpage
\section{Benchmark Results}\label{app:add_results}

\subsection{Qualitative Evaluation of All Learned Distributions}\label{app:sbi_plots}

One \textit{OneFlowSBI} model per task provides all four distributions in Bayes' theorem by changing only the mask: the posterior $p(\bm{\theta} \mid \mathbf{y}_{\mathrm{obs}})$, the likelihood $p(\mathbf{y} \mid \bm{\theta}_{\mathrm{true}})$, the prior $p(\bm{\theta})$ and the prior predictive $p(\mathbf{y})$. For each SBIBM task, we draw $10{,}000$ samples of each from a model trained on $30{,}000$ simulations and compare their one-dimensional marginals with those of reference samples, to show whether the model reproduces the location, spread and multimodal structure of each distribution.

\begin{figure}[htbp]
    \centering
    \includegraphics[width=\columnwidth]{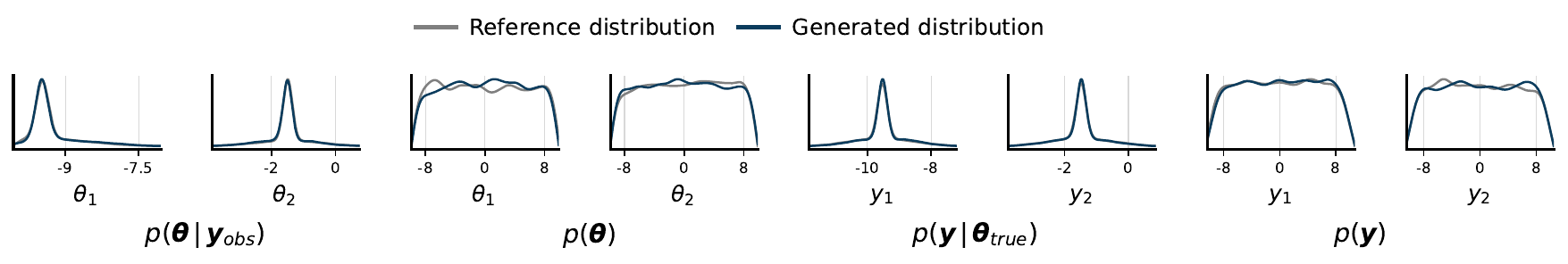}
    \caption{\textbf{Gaussian Mixture.}}
    \label{fig:distributions_gaussian_mixture}
\end{figure}

\begin{figure}[htbp]
    \centering
    \includegraphics[width=\columnwidth]{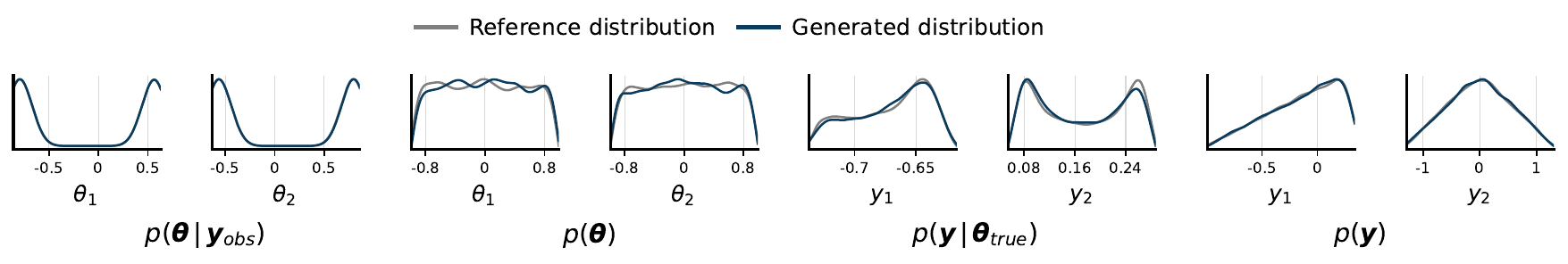}
    \caption{\textbf{Two Moons.}}
    \label{fig:distributions_two_moons}
\end{figure}

\begin{figure}[htbp]
    \centering
    \includegraphics[width=\columnwidth]{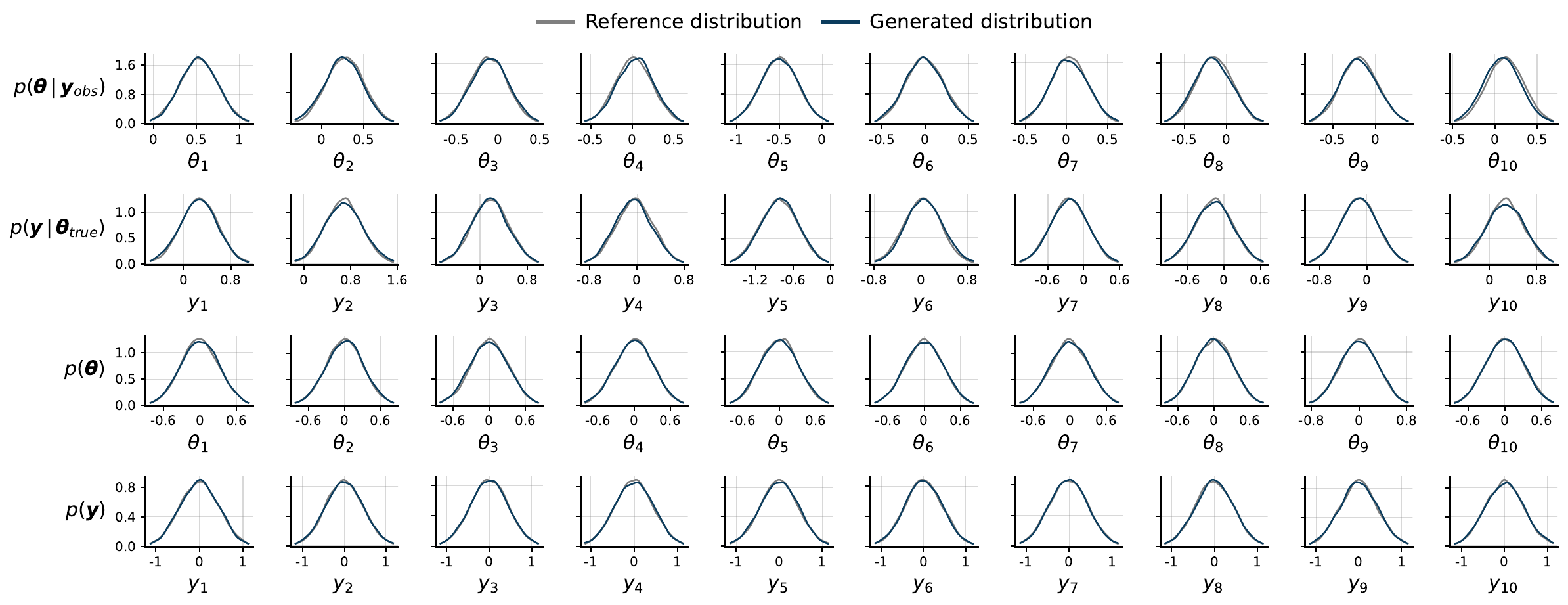}
    \caption{\textbf{Gaussian Linear.}}
    \label{fig:distributions_gaussian_linear}
\end{figure}

\begin{figure}[t]
    \centering
    \includegraphics[width=\columnwidth]{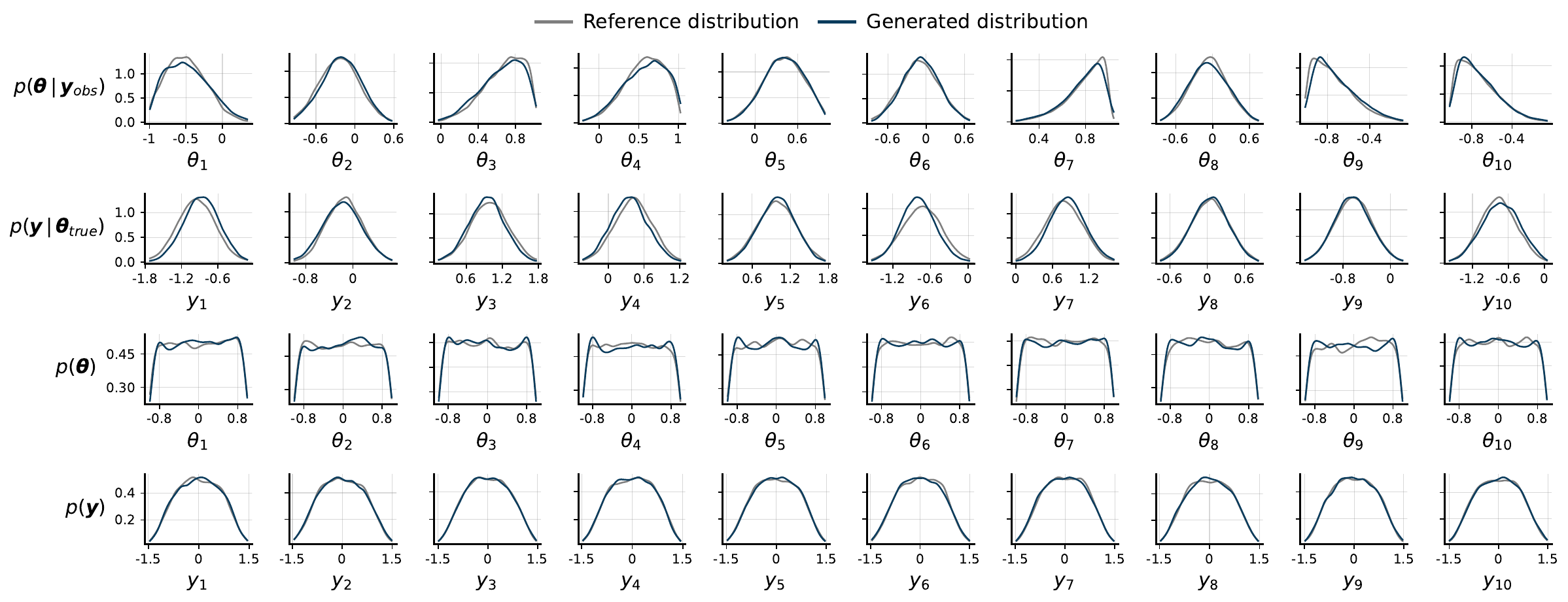}
    \caption{\textbf{Gaussian Linear Uniform.}}
    \label{fig:distributions_gaussian_linear_uniform}
\end{figure}

\begin{figure}[H]
    \centering
    \includegraphics[width=0.8\columnwidth]{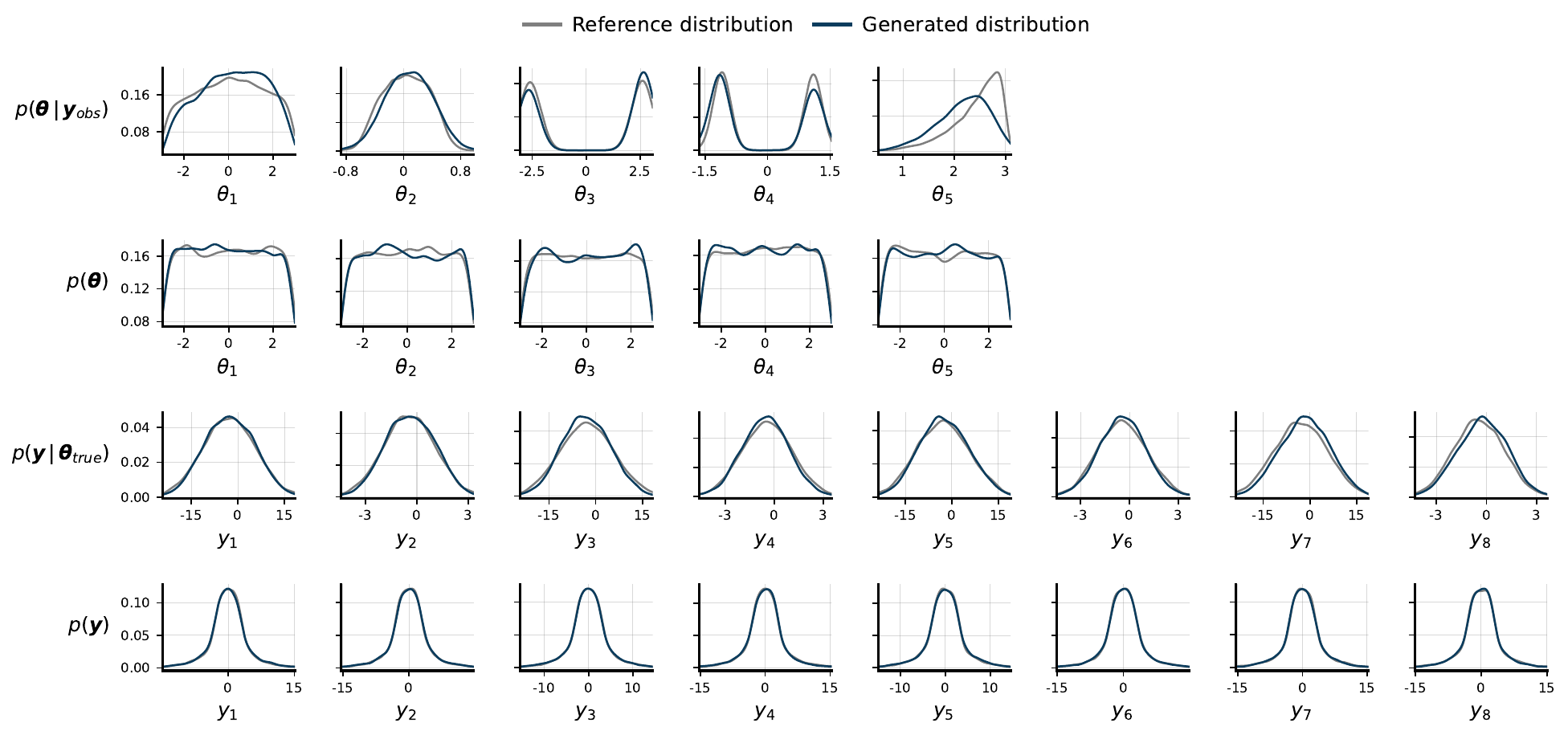}
    \caption{\textbf{SLCP.}}
    \label{fig:distributions_slcp}
\end{figure}

\begin{figure}[H]
    \centering
    \includegraphics[width=0.8\columnwidth]{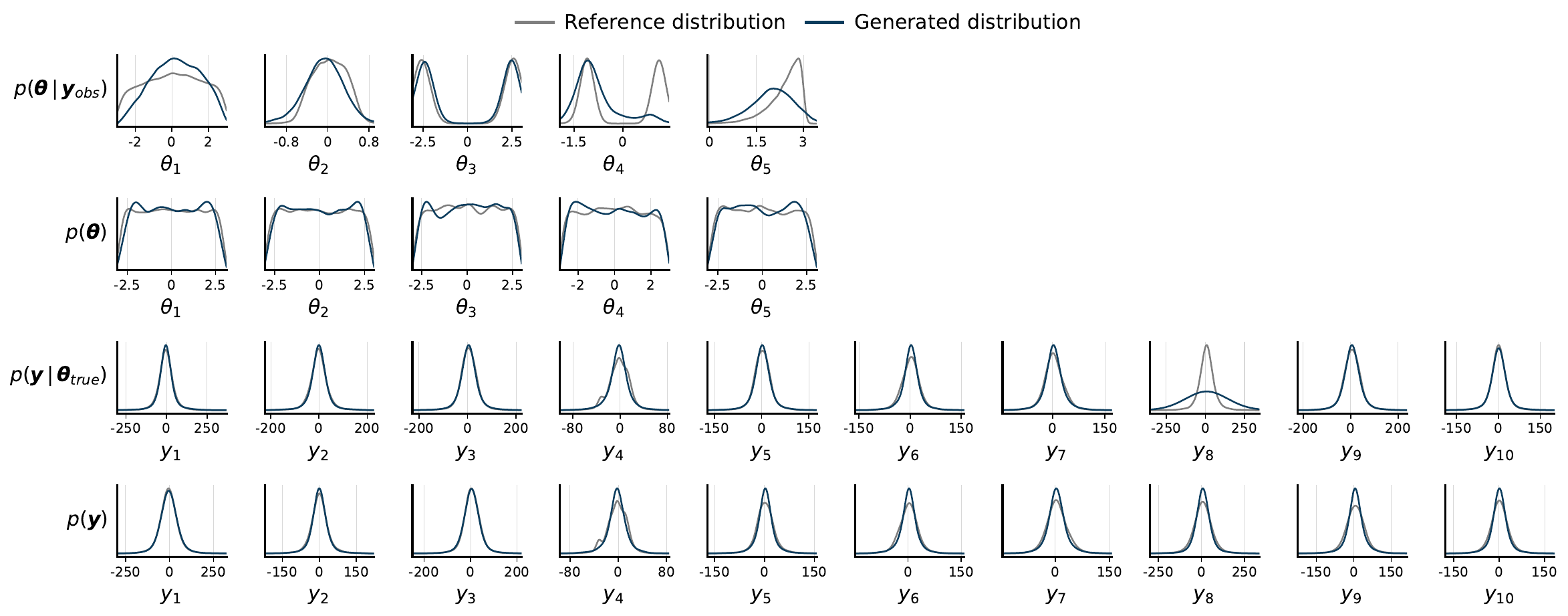}
    \caption{\textbf{SLCP Distractors.}}
    \label{fig:distributions_slcp_distractors}
\end{figure}

\begin{figure}[t]
    \centering
    \includegraphics[width=\columnwidth]{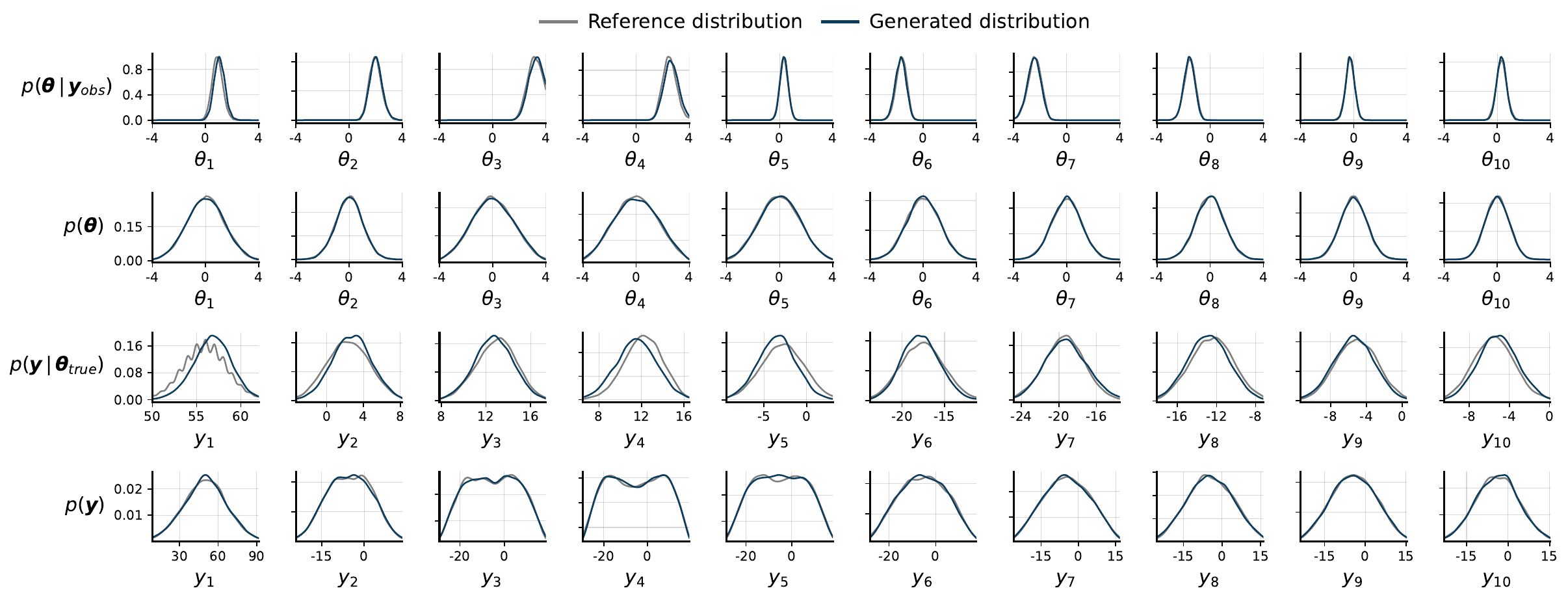}
    \caption{\textbf{Bernoulli GLM.}}
    \label{fig:distributions_bernoulli_glm}
\end{figure}

\begin{figure}[t]
    \centering
    \includegraphics[width=\columnwidth]{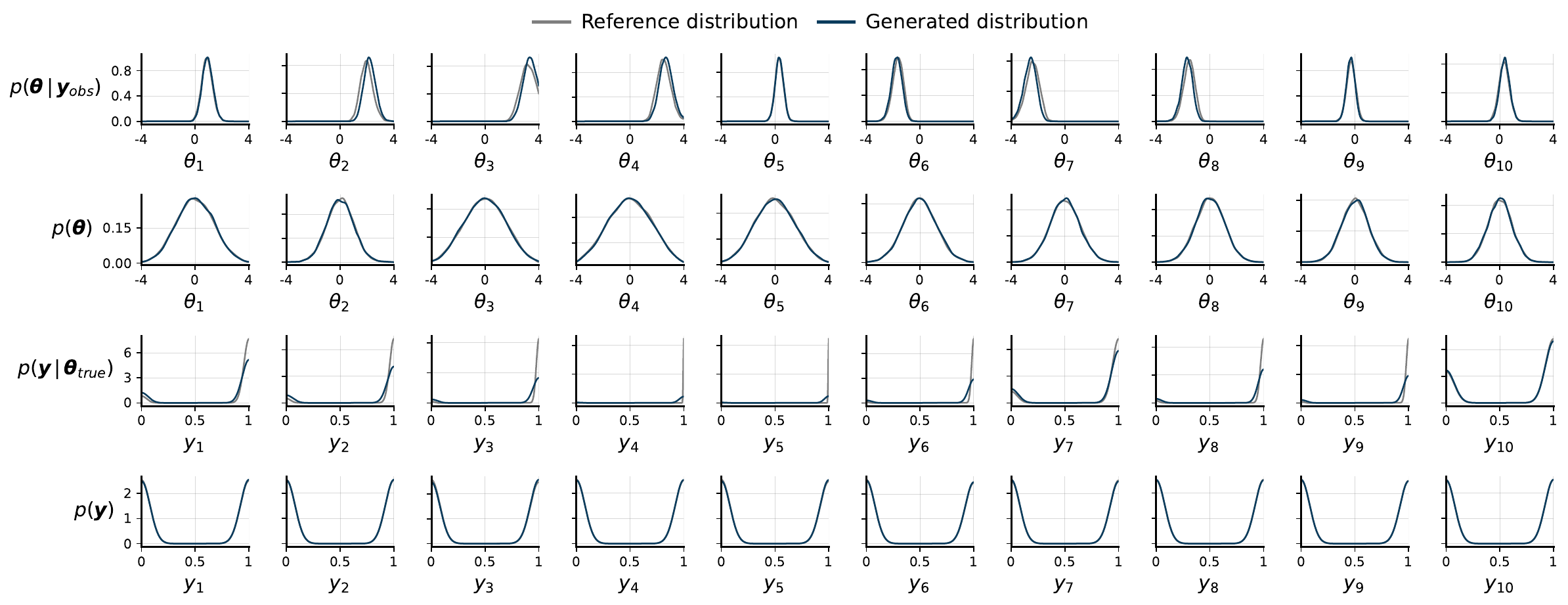}
    \caption{\textbf{Bernoulli GLM (raw).}}
    \label{fig:distributions_bernoulli_glm_raw}
\end{figure}

\begin{figure}[H]
    \centering
    \includegraphics[width=\columnwidth]{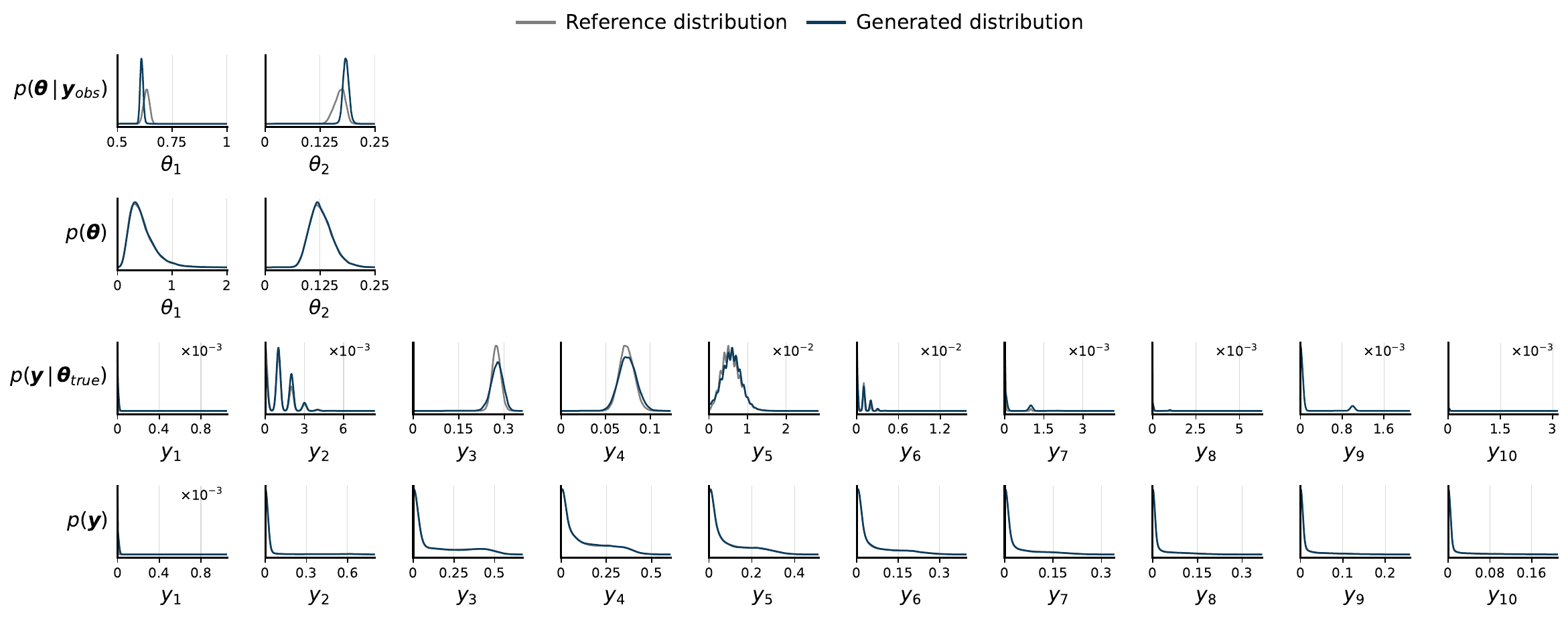}
    \caption{\textbf{SIR.}}
    \label{fig:distributions_sir}
\end{figure}

\begin{figure}[t]
    \centering
    \includegraphics[width=\columnwidth]{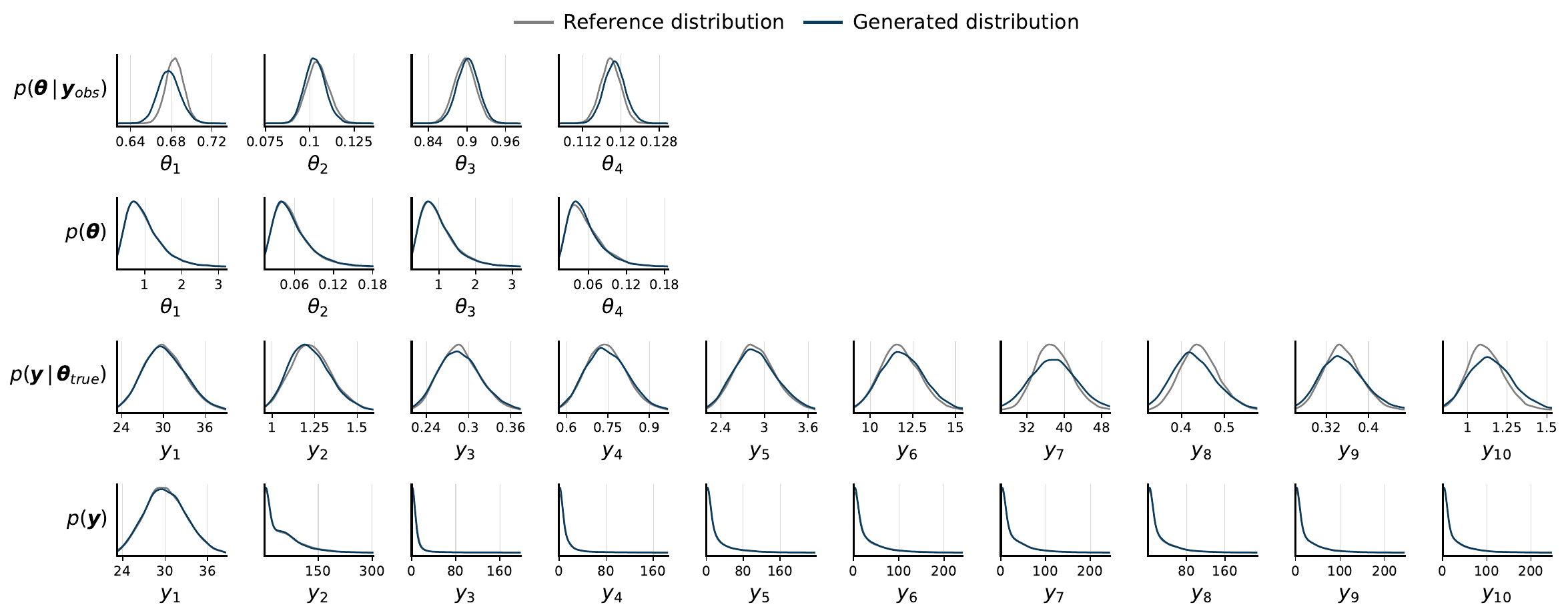}
    \caption{\textbf{Lotka--Volterra.}}
    \label{fig:distributions_lotka_volterra}
\end{figure}

\subsection{Quantitative Evaluation of All Learned Distributions}
\label{app:query_c2st}
We compare the three generalist methods, \textit{OneFlowSBI}, Simformer and ACE, on the posterior, likelihood, prior and prior predictive with C2ST (Figure~\ref{fig:query_c2st}). Each method answers all queries of a task with one model trained on $30{,}000$ simulations. Posteriors condition on the ten SBIBM reference observations and likelihoods on the corresponding true parameters; the prior and prior predictive are evaluated with ten sample sets per model. ACE is evaluated on the posterior and likelihood only, since its training protocol excludes unconditional distributions.

On the posterior, no method is consistently best, and the differences are mostly small, with ACE deviating most on a few tasks. All three methods struggle with the likelihood where observations are discrete or contain many uninformative dimensions. On this query, \textit{OneFlowSBI} is as accurate as Simformer or more accurate on every task, by a wide margin on some, while ACE matches or slightly exceeds it on some tasks and falls well behind on others. For the marginals, \textit{OneFlowSBI} and Simformer both approximate the prior closely on all tasks but one, where Simformer deviates. On the prior predictive, \textit{OneFlowSBI} is as accurate as Simformer or more accurate on nine of ten tasks, with its largest errors where observations are discrete or heavy-tailed.

\begin{figure}[H]
  \centering
  \includegraphics[width=0.9\columnwidth]{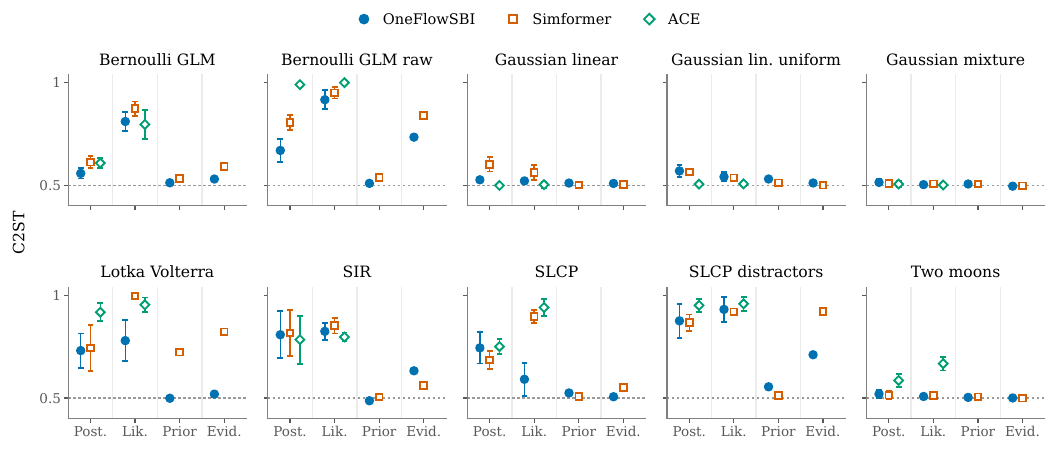}
  \caption{C2ST of each learned distribution at $30{,}000$ simulations
  (lower is better, $0.5$ is optimal). Bars show $\pm$ one standard
  deviation over the ten reference observations, or over ten sample sets
  for the prior and prior predictive.}
  \label{fig:query_c2st}
\end{figure}

\subsection{Quantitative Evaluation of Posterior Accuracy}
\label{app:c2st_tables}

Table~\ref{tab:c2st_main} reports the posterior C2ST of \textit{OneFlowSBI}, NPE, Simformer and ACE on the ten SBIBM tasks at each budget. On Lotka--Volterra, whose posterior is very concentrated, the samples of NPE do not overlap the reference posterior, so its C2ST saturates at $1.000$ at every budget.

\begin{table}[!htbp]
\centering
\renewcommand{\arraystretch}{1.25}
\setlength{\tabcolsep}{5pt}
\caption{Posterior C2ST across simulation budgets; closer to
0.5 is better. OneFlowSBI entries report the mean $\pm$ standard
deviation of five per-run scores, each averaged over ten
observations. Results are rounded to three
decimal places. Best means are \textbf{bold} and second-best means
are \underline{underlined}, with ties sharing the same formatting.}
\label{tab:c2st_main}

\begin{subtable}[t]{\textwidth}
\centering
\begin{adjustbox}{max width=\textwidth}
\begin{tabular}{l ccc ccc ccc}
\toprule
Method
& \multicolumn{3}{c}{SLCP Distractors}
& \multicolumn{3}{c}{Bernoulli GLM (raw)}
& \multicolumn{3}{c}{Lotka--Volterra} \\
\cmidrule(lr){2-4}\cmidrule(lr){5-7}\cmidrule(lr){8-10}
& 10k & 20k & 30k
& 10k & 20k & 30k
& 10k & 20k & 30k \\
\midrule
NPE
& \cTwoST{$\underline{0.937}$}{$0.036$}
& \cTwoST{$\underline{0.909}$}{$0.045$}
& \cTwoST{$0.897$}{$0.049$}
& \cTwoST{$\mathbf{0.784}$}{$0.034$}
& \cTwoST{$\mathbf{0.729}$}{$0.042$}
& \cTwoST{$\underline{0.697}$}{$0.042$}
& \cTwoST{$1.000$}{$0.000$}
& \cTwoST{$1.000$}{$0.000$}
& \cTwoST{$1.000$}{$0.000$} \\
Simformer
& \cTwoST{$\mathbf{0.895}$}{$0.044$}
& \cTwoST{$\mathbf{0.863}$}{$0.049$}
& \cTwoST{$\mathbf{0.867}$}{$0.038$}
& \cTwoST{$0.970$}{$0.012$}
& \cTwoST{$0.868$}{$0.044$}
& \cTwoST{$0.806$}{$0.036$}
& \cTwoST{$\underline{0.831}$}{$0.067$}
& \cTwoST{$\underline{0.811}$}{$0.067$}
& \cTwoST{$\underline{0.743}$}{$0.105$} \\
ACE
& \cTwoST{$0.944$}{$0.009$}
& \cTwoST{$0.947$}{$0.007$}
& \cTwoST{$0.941$}{$0.009$}
& \cTwoST{$\underline{0.943}$}{$0.014$}
& \cTwoST{$0.969$}{$0.033$}
& \cTwoST{$0.995$}{$0.003$}
& \cTwoST{$0.918$}{$0.008$}
& \cTwoST{$0.917$}{$0.012$}
& \cTwoST{$0.917$}{$0.004$} \\
\rowcolor{gray!10}
OneFlowSBI
& \cTwoST{$0.995$}{$0.002$}
& \cTwoST{$0.941$}{$0.015$}
& \cTwoST{$\underline{0.887}$}{$0.007$}
& \cTwoST{$0.997$}{$0.001$}
& \cTwoST{$\underline{0.827}$}{$0.037$}
& \cTwoST{$\mathbf{0.682}$}{$0.023$}
& \cTwoST{$\mathbf{0.739}$}{$0.008$}
& \cTwoST{$\mathbf{0.724}$}{$0.010$}
& \cTwoST{$\mathbf{0.720}$}{$0.011$} \\
\bottomrule
\end{tabular}
\end{adjustbox}
\end{subtable}

\vspace{0.2cm}

\begin{subtable}[t]{\textwidth}
\centering
\setlength{\tabcolsep}{4pt}
\begin{adjustbox}{max width=\textwidth}
\begin{tabular}{l ccc ccc ccc ccc}
\toprule
Method
& \multicolumn{3}{c}{Gaussian Mixture}
& \multicolumn{3}{c}{Gaussian Linear}
& \multicolumn{3}{c}{Gaussian Linear Uniform}
& \multicolumn{3}{c}{SLCP} \\
\cmidrule(lr){2-4}\cmidrule(lr){5-7}
\cmidrule(lr){8-10}\cmidrule(lr){11-13}
& 10k & 20k & 30k
& 10k & 20k & 30k
& 10k & 20k & 30k
& 10k & 20k & 30k \\
\midrule
NPE
& \cTwoST{$0.582$}{$0.025$}
& \cTwoST{$0.545$}{$0.013$}
& \cTwoST{$0.559$}{$0.056$}
& \cTwoST{$\underline{0.583}$}{$0.029$}
& \cTwoST{$0.555$}{$0.026$}
& \cTwoST{$0.547$}{$0.023$}
& \cTwoST{$\underline{0.600}$}{$0.036$}
& \cTwoST{$\underline{0.565}$}{$0.018$}
& \cTwoST{$\underline{0.561}$}{$0.024$}
& \cTwoST{$0.881$}{$0.040$}
& \cTwoST{$0.850$}{$0.049$}
& \cTwoST{$0.837$}{$0.056$} \\
Simformer
& \cTwoST{$0.518$}{$0.010$}
& \cTwoST{$\underline{0.513}$}{$0.013$}
& \cTwoST{$\mathbf{0.509}$}{$0.008$}
& \cTwoST{$0.639$}{$0.053$}
& \cTwoST{$0.605$}{$0.033$}
& \cTwoST{$0.602$}{$0.034$}
& \cTwoST{$0.623$}{$0.034$}
& \cTwoST{$0.602$}{$0.032$}
& \cTwoST{$0.565$}{$0.016$}
& \cTwoST{$\mathbf{0.771}$}{$0.055$}
& \cTwoST{$\mathbf{0.729}$}{$0.062$}
& \cTwoST{$\mathbf{0.685}$}{$0.043$} \\
ACE
& \cTwoST{$\mathbf{0.509}$}{$0.002$}
& \cTwoST{$\mathbf{0.507}$}{$0.002$}
& \cTwoST{$\mathbf{0.509}$}{$0.009$}
& \cTwoST{$\mathbf{0.513}$}{$0.004$}
& \cTwoST{$\mathbf{0.502}$}{$0.003$}
& \cTwoST{$\mathbf{0.501}$}{$0.001$}
& \cTwoST{$\mathbf{0.515}$}{$0.004$}
& \cTwoST{$\mathbf{0.506}$}{$0.004$}
& \cTwoST{$\mathbf{0.505}$}{$0.001$}
& \cTwoST{$\underline{0.787}$}{$0.025$}
& \cTwoST{$\underline{0.803}$}{$0.029$}
& \cTwoST{$0.784$}{$0.022$} \\
\rowcolor{gray!10}
OneFlowSBI
& \cTwoST{$\underline{0.515}$}{$0.003$}
& \cTwoST{$0.516$}{$0.005$}
& \cTwoST{$\underline{0.512}$}{$0.005$}
& \cTwoST{$0.715$}{$0.035$}
& \cTwoST{$\underline{0.552}$}{$0.006$}
& \cTwoST{$\underline{0.525}$}{$0.001$}
& \cTwoST{$0.816$}{$0.036$}
& \cTwoST{$0.613$}{$0.007$}
& \cTwoST{$0.578$}{$0.005$}
& \cTwoST{$0.929$}{$0.013$}
& \cTwoST{$0.807$}{$0.020$}
& \cTwoST{$\underline{0.745}$}{$0.012$} \\
\bottomrule
\end{tabular}
\end{adjustbox}
\end{subtable}

\vspace{0.2cm}

\begin{subtable}[t]{\textwidth}
\centering
\begin{adjustbox}{max width=\textwidth}
\begin{tabular}{l ccc ccc ccc}
\toprule
Method
& \multicolumn{3}{c}{Two Moons}
& \multicolumn{3}{c}{Bernoulli GLM}
& \multicolumn{3}{c}{SIR} \\
\cmidrule(lr){2-4}\cmidrule(lr){5-7}\cmidrule(lr){8-10}
& 10k & 20k & 30k
& 10k & 20k & 30k
& 10k & 20k & 30k \\
\midrule
NPE
& \cTwoST{$0.592$}{$0.056$}
& \cTwoST{$0.581$}{$0.050$}
& \cTwoST{$0.572$}{$0.052$}
& \cTwoST{$0.696$}{$0.030$}
& \cTwoST{$0.635$}{$0.027$}
& \cTwoST{$0.645$}{$0.035$}
& \cTwoST{$0.824$}{$0.128$}
& \cTwoST{$0.824$}{$0.120$}
& \cTwoST{$0.828$}{$0.122$} \\
Simformer
& \cTwoST{$\underline{0.531}$}{$0.025$}
& \cTwoST{$\mathbf{0.518}$}{$0.018$}
& \cTwoST{$\mathbf{0.514}$}{$0.018$}
& \cTwoST{$\underline{0.657}$}{$0.022$}
& \cTwoST{$0.647$}{$0.030$}
& \cTwoST{$0.613$}{$0.027$}
& \cTwoST{$0.819$}{$0.106$}
& \cTwoST{$0.814$}{$0.114$}
& \cTwoST{$0.817$}{$0.107$} \\
ACE
& \cTwoST{$0.592$}{$0.015$}
& \cTwoST{$0.590$}{$0.017$}
& \cTwoST{$0.591$}{$0.005$}
& \cTwoST{$\mathbf{0.615}$}{$0.005$}
& \cTwoST{$\underline{0.600}$}{$0.003$}
& \cTwoST{$\underline{0.601}$}{$0.005$}
& \cTwoST{$\mathbf{0.787}$}{$0.009$}
& \cTwoST{$\mathbf{0.786}$}{$0.004$}
& \cTwoST{$\mathbf{0.788}$}{$0.008$} \\
\rowcolor{gray!10}
OneFlowSBI
& \cTwoST{$\mathbf{0.524}$}{$0.003$}
& \cTwoST{$\underline{0.519}$}{$0.004$}
& \cTwoST{$\underline{0.519}$}{$0.005$}
& \cTwoST{$0.679$}{$0.007$}
& \cTwoST{$\mathbf{0.584}$}{$0.003$}
& \cTwoST{$\mathbf{0.561}$}{$0.004$}
& \cTwoST{$\underline{0.807}$}{$0.007$}
& \cTwoST{$\underline{0.800}$}{$0.006$}
& \cTwoST{$\underline{0.801}$}{$0.007$} \\
\bottomrule
\end{tabular}
\end{adjustbox}
\end{subtable}

\end{table}

\subsection{Calibration}\label{app:calibration}
We assess calibration with simulation based calibration (SBC), which requires no reference distribution~\citep{talts2018validating}. For one model per task trained on $30{,}000$ simulations, we draw $1{,}000$ new pairs $(\bm{\theta}^{*}, \mathbf{y})$ from the prior and the simulator. For the posterior, we compute the rank of each coordinate of $\bm{\theta}^{*}$ among $L = 399$ samples from $p(\bm{\theta} \mid \mathbf{y})$, and for the likelihood, the rank of each coordinate of $\mathbf{y}$ among $L$ samples from $p(\mathbf{y} \mid \bm{\theta}^{*})$; for a calibrated model, each rank is uniform on $\{0, \dots, L\}$. Equivalently, the truth lies in the central interval of level $q$ with probability $q$, so coverage follows the diagonal.

The posterior curves follow the diagonal on most tasks (Figure~\ref{fig:sbc_posterior}); they fall slightly below it on Bernoulli GLM raw and clearly below it on SLCP distractors, where the posterior is overconfident. The likelihood curves also follow the diagonal on most tasks (Figure~\ref{fig:sbc_likelihood}), but lie above it on Lotka--Volterra and SIR, where the likelihood is conservative, and on both sides of it on SLCP distractors, where some outputs are conservative and others overconfident.

We attribute these deviations mainly to the cost of learning the joint distribution with one model. They concern mostly the likelihood, which on Lotka--Volterra, SIR and SLCP distractors must generate an observation with several times more dimensions than the parameters, yet receives only part of the training masks, consistent with the amortization gap measured in Section~\ref{sec:generality_ablation}. On SLCP distractors, the model must also generate $92$ heavy-tailed noise coordinates that a posterior estimator only takes as input, and this is the one task with a clearly overconfident posterior (Section~\ref{sec:limit}). Part of the deviations may also stem from the shared configuration, which keeps the evaluation consistent: one architecture and training protocol serve all tasks, from Gaussian models to discrete or heavy-tailed observations (Appendix~\ref{app:sbi_hyperparams}).

\begin{figure*}[h]
  \centering
  \includegraphics[width=0.75\textwidth]{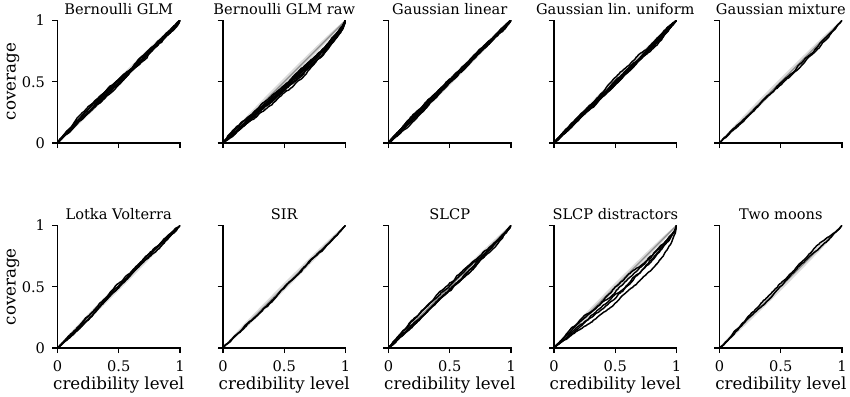}
\caption{\textbf{Posterior calibration.} Coverage of central posterior intervals against the credibility level, one curve per parameter, for one model per task at $30{,}000$ simulations. The gray band marks the $95\%$ range of a calibrated model.}
  \label{fig:sbc_posterior}
\end{figure*}

\begin{figure*}[h]
  \centering
  \includegraphics[width=0.75\textwidth]{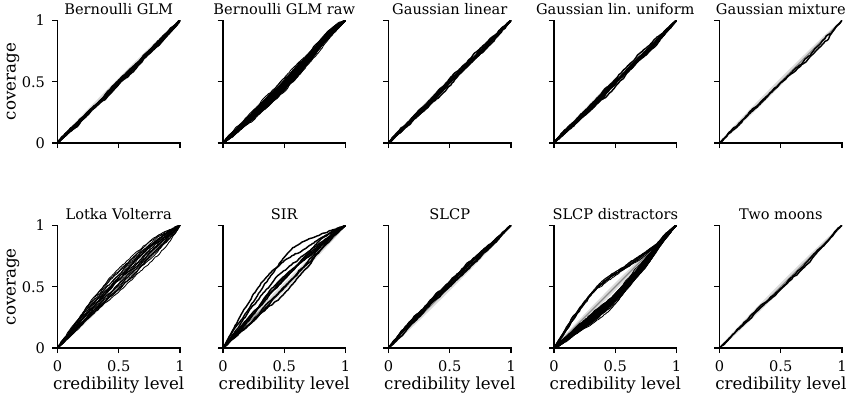}
  \caption{\textbf{Likelihood calibration.} As
  Figure~\ref{fig:sbc_posterior}, for central intervals of
  $p(\mathbf{y} \mid \bm{\theta}^{*})$, one curve per observation
  coordinate.}
  \label{fig:sbc_likelihood}
\end{figure*}

\newpage
\subsection{Training and Sampling Cost}\label{app:train_sample_times}

The three generalists sample in different ways, so their costs scale differently. \textit{OneFlowSBI} integrates an ODE with an adaptive solver, and each network evaluation processes the full joint state; changing the query changes only the mask, so the cost per evaluation is the same for every query, and the sampling time depends on how many evaluations the solver needs, that is, on how complex the learned flow is rather than on $d$. Simformer integrates its reverse SDE with $500$ steps, each attending over all $d$ variables, so its cost is the same for every query but grows with $d$. ACE needs no solver: it draws each generated coordinate from a one-dimensional mixture in a single network pass, so its cost grows with the number of coordinates a query generates. The methods also differ in implementation, with Simformer in JAX and the others in PyTorch, and in their sampler settings, which are not tuned for speed: Simformer's authors report that $50$ steps can suffice~\citep{gloecklerall}, and \textit{OneFlowSBI} uses the adaptive \texttt{dopri5} solver, which evaluates the network six times per step against once for the Euler method, so cheaper solvers could also be explored. The times below therefore show how the cost of each design scales, rather than providing a controlled comparison.

Table~\ref{tab:ofs_cost} reports the cost of \textit{OneFlowSBI}: training takes under ten minutes at every budget, since the number of steps is fixed, and $10{,}000$ samples of any distribution take at most two seconds; the posterior takes longest on SLCP distractors and shortest on Bernoulli GLM raw, although both have about $100$ variables. Table~\ref{tab:cost_vs_simformer} compares the posterior and likelihood, which all three methods can sample. We do not compare training times, which follow each method's default number of steps: a fixed $50{,}000$ for \textit{OneFlowSBI} and $50{,}000$ for ACE, and three times the number of simulations for Simformer. In the worst case, on the two tasks with $100$-dimensional observations, Simformer takes about $31$ seconds for either query, while \textit{OneFlowSBI} stays within two seconds; with $50$ steps, Simformer would be faster on the eight smaller tasks, but not on these two. ACE is the fastest on most tasks, because it needs no solver and these queries generate few coordinates, but on the same two tasks its likelihood takes over $4$ seconds, longer than that of \textit{OneFlowSBI}; its training protocol also excludes the prior, prior predictive and joint, and it is less accurate on several tasks (Appendix~\ref{app:query_c2st}). Since these tasks have at most $110$ variables, all three methods are practical here; differences in cost matter at the scale of our case studies (Section~\ref{sec:experiments}).

\begin{table}[H]
\centering
\small
\setlength{\tabcolsep}{4pt}
\renewcommand{\arraystretch}{1.1}
\caption{\textbf{Training and sampling cost of \textit{OneFlowSBI}.} Training time in minutes at each simulation budget, and seconds per $10{,}000$ samples of each distribution at $30{,}000$ simulations; mean $\pm$ standard deviation over five runs. $^{\dagger}15{,}000$ training steps.}
\label{tab:ofs_cost}
\begin{adjustbox}{max width=\textwidth}
\begin{tabular}{l ccc cccc}
\toprule
& \multicolumn{3}{c}{Training (min)} & \multicolumn{4}{c}{Sampling (s)} \\
\cmidrule(lr){2-4}\cmidrule(lr){5-8}
Task & $10$k & $20$k & $30$k
& $p(\bm{\theta}\mid\mathbf{y})$ & $p(\mathbf{y}\mid\bm{\theta})$
& $p(\bm{\theta})$ & $p(\mathbf{y})$ \\
\midrule
Gaussian mixture
& \cTwoST{$8.39$}{$0.29$} & \cTwoST{$7.56$}{$0.29$} & \cTwoST{$7.53$}{$0.28$}
& \cTwoST{$1.15$}{$0.04$} & \cTwoST{$1.07$}{$0.08$} & \cTwoST{$0.76$}{$0.07$} & \cTwoST{$0.77$}{$0.06$} \\
Two moons
& \cTwoST{$7.43$}{$0.20$} & \cTwoST{$7.53$}{$0.30$} & \cTwoST{$7.52$}{$0.08$}
& \cTwoST{$1.45$}{$0.14$} & \cTwoST{$1.30$}{$0.10$} & \cTwoST{$0.94$}{$0.13$} & \cTwoST{$0.95$}{$0.08$} \\
SIR
& \cTwoST{$6.97$}{$0.23$} & \cTwoST{$8.18$}{$0.92$} & \cTwoST{$6.78$}{$0.14$}
& \cTwoST{$1.22$}{$0.09$} & \cTwoST{$1.19$}{$0.07$} & \cTwoST{$0.78$}{$0.06$} & \cTwoST{$1.10$}{$0.03$} \\
SLCP
& \cTwoST{$7.43$}{$0.20$} & \cTwoST{$8.99$}{$0.23$} & \cTwoST{$7.69$}{$0.18$}
& \cTwoST{$1.43$}{$0.12$} & \cTwoST{$1.33$}{$0.04$} & \cTwoST{$0.84$}{$0.08$} & \cTwoST{$0.98$}{$0.04$} \\
Bernoulli GLM
& \cTwoST{$8.55$}{$0.15$} & \cTwoST{$7.62$}{$0.06$} & \cTwoST{$7.23$}{$0.20$}
& \cTwoST{$1.03$}{$0.06$} & \cTwoST{$0.96$}{$0.06$} & \cTwoST{$0.85$}{$0.02$} & \cTwoST{$0.86$}{$0.05$} \\
Gaussian linear
& \cTwoST{$7.25$}{$0.22$} & \cTwoST{$7.53$}{$0.28$} & \cTwoST{$7.14$}{$0.21$}
& \cTwoST{$1.08$}{$0.05$} & \cTwoST{$1.02$}{$0.06$} & \cTwoST{$0.79$}{$0.03$} & \cTwoST{$0.75$}{$0.04$} \\
Gaussian linear uniform
& \cTwoST{$8.46$}{$0.23$} & \cTwoST{$7.76$}{$0.19$} & \cTwoST{$7.54$}{$0.18$}
& \cTwoST{$1.11$}{$0.04$} & \cTwoST{$0.99$}{$0.12$} & \cTwoST{$0.84$}{$0.08$} & \cTwoST{$0.79$}{$0.06$} \\
Lotka--Volterra
& \cTwoST{$6.69$}{$0.10$} & \cTwoST{$7.01$}{$0.13$} & \cTwoST{$6.87$}{$0.43$}
& \cTwoST{$1.02$}{$0.19$} & \cTwoST{$0.94$}{$0.07$} & \cTwoST{$0.60$}{$0.03$} & \cTwoST{$0.85$}{$0.03$} \\
SLCP distractors$^{\dagger}$
& \cTwoST{$2.17$}{$0.06$} & \cTwoST{$2.16$}{$0.07$} & \cTwoST{$2.36$}{$0.05$}
& \cTwoST{$2.00$}{$0.11$} & \cTwoST{$1.37$}{$0.08$} & \cTwoST{$1.02$}{$0.07$} & \cTwoST{$1.39$}{$0.17$} \\
Bernoulli GLM raw
& \cTwoST{$7.92$}{$0.22$} & \cTwoST{$8.19$}{$0.09$} & \cTwoST{$6.82$}{$0.17$}
& \cTwoST{$0.86$}{$0.12$} & \cTwoST{$0.72$}{$0.04$} & \cTwoST{$0.72$}{$0.02$} & \cTwoST{$0.80$}{$0.02$} \\
\bottomrule
\end{tabular}
\end{adjustbox}
\end{table}

\begin{table}[H]
\centering
\small
\setlength{\tabcolsep}{4pt}
\caption{\textbf{Sampling cost of the three generalists at $30{,}000$ simulations.} Seconds per $10{,}000$ posterior and likelihood samples, with tasks sorted by the number of variables.}
\label{tab:cost_vs_simformer}
\begin{tabular}{lc ccc ccc}
\toprule
& & \multicolumn{3}{c}{Posterior (s)} & \multicolumn{3}{c}{Likelihood (s)} \\
\cmidrule(lr){3-5}\cmidrule(lr){6-8}
Task & $d_\theta + d_y$ & Ours & Simformer & ACE & Ours & Simformer & ACE \\
\midrule
Gaussian mixture         & $2+2$    & 1.15 & 2.96  & 0.03 & 1.07 & 2.91  & 0.03 \\
Two moons                & $2+2$    & 1.45 & 3.00  & 0.03 & 1.30 & 2.94  & 0.02 \\
SIR                      & $2+10$   & 1.22 & 4.02  & 0.03 & 1.19 & 4.05  & 0.09 \\
SLCP                     & $5+8$    & 1.43 & 4.20  & 0.08 & 1.33 & 4.42  & 0.08 \\
Bernoulli GLM            & $10+10$  & 1.03 & 5.52  & 0.18 & 0.96 & 5.34  & 0.12 \\
Gaussian linear          & $10+10$  & 1.08 & 5.55  & 0.21 & 1.02 & 5.37  & 0.14 \\
Gaussian linear uniform  & $10+10$  & 1.11 & 5.37  & 0.18 & 0.99 & 5.60  & 0.12 \\
Lotka--Volterra           & $4+20$   & 1.02 & 6.01  & 0.09 & 0.94 & 6.32  & 0.24 \\
SLCP distractors         & $5+100$  & 2.00 & 31.00 & 0.39 & 1.37 & 31.02 & 4.04 \\
Bernoulli GLM raw        & $10+100$ & 0.86 & 31.20 & 0.86 & 0.72 & 31.24 & 4.37 \\
\bottomrule
\end{tabular}
\end{table}

\newpage

\section{Additional Case Study Results}

\subsection{Bayesian Denoising: Comparison with Baselines}
\label{app:add_results_denoising}

We compare posteriors from five models, chosen to separate the effect of learning many queries from that of the backbone. NPE is a posterior specialist. Simformer is trained on all conditionals and, with the same architecture, on the posterior alone, so the difference between the two shows the cost of learning every conditional. The two versions of \textit{OneFlowSBI} are both trained on all queries and differ only in the backbone: a convolutional network, the natural backbone for images, or the MLP of the SBIBM tasks (Appendix~\ref{app:sbi_hyperparams}) on the flattened joint state $\mathbf{z} \in \mathbb{R}^{1568}$, widened to $2{,}048$ hidden units, since a hidden layer narrower than the $1{,}568$ outputs would confine the velocity to a lower-dimensional subspace. The MLP has $92$M trainable parameters, against $1.4$M for the convolutional network. Training details are in Appendix~\ref{app:hyperparam_image}.

In Figure~\ref{fig:denoising_baselines}, the baselines recover the coarse shape of each garment, but their means are blurred and their uncertainty spreads over the whole garment, and for NPE also over the background. Both versions of \textit{OneFlowSBI} recover sharper images with lower uncertainty, the convolutional one most clearly. Over $100$ test images, \textit{OneFlowSBI} with the convolutional network reduces both the RMSE of the posterior mean and the energy score to less than half those of the best baseline in each of five training runs (Table~\ref{tab:denoising_test}). Training Simformer on the posterior alone improves it only to about the level of NPE, so learning every conditional costs some accuracy, yet most of the gap to \textit{OneFlowSBI} remains. With the MLP, \textit{OneFlowSBI} is still more accurate than every baseline, including the two posterior-only models, so its advantage does not rest on the convolutional backbone alone. The convolutional network, which encoding pixel roles as input channels allows, lowers the RMSE and energy score by more than a third, with $1.4$M parameters instead of $92$M.

\begin{figure*}[!htbp]
  \centering
  \includegraphics[width=0.8\textwidth]{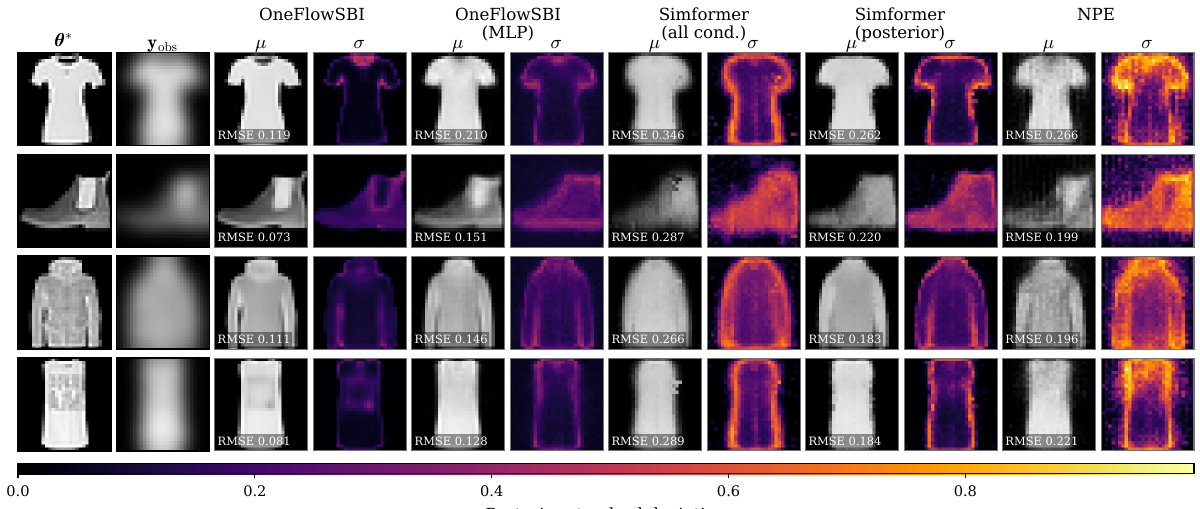}
  \caption{\textbf{Bayesian denoising: comparison with baselines.} Posterior
  mean ($\mu$), with its RMSE, and standard deviation ($\sigma$) for four test
  images. The colour bar gives the posterior standard deviation.}
  \label{fig:denoising_baselines}
\end{figure*}

\begin{table}[H]
\centering
\small
\caption{\textbf{Bayesian denoising: posterior accuracy on $100$ test images.} RMSE of the posterior mean and energy score~\citep{gneiting2007strictly}, which also accounts for the posterior spread, from $100$ posterior samples per image (lower is better). For \textit{OneFlowSBI} with the convolutional network, mean $\pm$ standard deviation over five training runs; for the other rows, one run, mean $\pm$ standard error over images.}
\label{tab:denoising_test}
\begin{tabular}{lcc}
\toprule
Method & RMSE & Energy score \\
\midrule
\textit{OneFlowSBI}          & $\mathbf{0.111 \pm 0.0002}$ & $\mathbf{2.20 \pm 0.005}$ \\
\textit{OneFlowSBI} (MLP)    & $0.181 \pm 0.006$ & $3.59 \pm 0.12$ \\
\midrule
Simformer (all conditionals) & $0.326 \pm 0.006$ & $6.48 \pm 0.12$ \\
Simformer (posterior)        & $0.245 \pm 0.007$ & $4.88 \pm 0.14$ \\
NPE                          & $0.247 \pm 0.007$ & $5.42 \pm 0.13$ \\
\bottomrule
\end{tabular}
\end{table}

\subsubsection{Bayesian Denoising: Calibration}
\label{app:denoising_calibration}

We test calibration over $1{,}000$ test pairs with TARP~\citep{tarp}, which checks the full image at once rather than pixel by pixel, using $399$ samples per pair and, as reference points, the true values of other test pairs (Figure~\ref{fig:image_tarp}). The posterior lies close to the diagonal, slightly below it, indicating mild overconfidence. For the likelihood, the test concerns only the noise: given $\bm{\theta}^{*}$, the observation is almost deterministic and varies only through blurred Poisson noise, which is small and vanishes where the image is dark. Since this noise is small compared with the distance to the reference point, the test effectively ranks the true observation among the samples along the direction of that point, and the curve, below the diagonal at low credibility levels and above it at high ones, indicates samples that vary slightly more than the simulator's, that is, a conservative likelihood. As on Lotka--Volterra, the likelihood is conservative where the simulator's spread is small, while its mean matches the simulator closely (Figure~\ref{fig:image_queries}, row 3).

\begin{figure}[!htbp]
  \centering
  \includegraphics[width=0.4\textwidth]{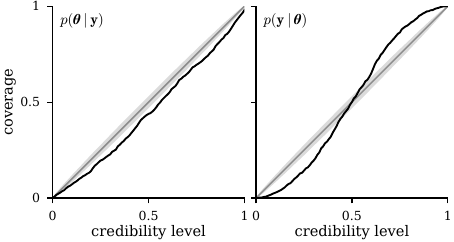}
  \caption{\textbf{Bayesian denoising: calibration.} TARP coverage of the posterior and likelihood against the credibility level, over $1{,}000$ test pairs; the gray band marks the $95\%$ range of a calibrated model.}
  \label{fig:image_tarp}
\end{figure}

\subsubsection{Missing observation block}\label{app:denoising_block}Figure~\ref{fig:denoising_block_baselines} shows the posterior $p(\bm{\theta} \mid \mathbf{y}_{\mathrm{obs}}^{\mathrm{partial}})$ when the central $8 \times 8$ block of the observation is missing, as in row~5 of Figure~\ref{fig:image_queries}. We compare with Simformer trained on all conditionals, which, having no absent role, generates the missing pixels jointly with the image and discards them; this marginalizes them out, so both methods target the same posterior. Since the block removes information about the image behind it, the posterior mean should be smooth and the uncertainty highest in that region. \textit{OneFlowSBI} shows both and has a lower error on every image, whereas the uncertainty of Simformer stays spread over the garment and does not respond to the block.

\begin{figure}[h]
  \centering
  \includegraphics[width=0.5\textwidth]{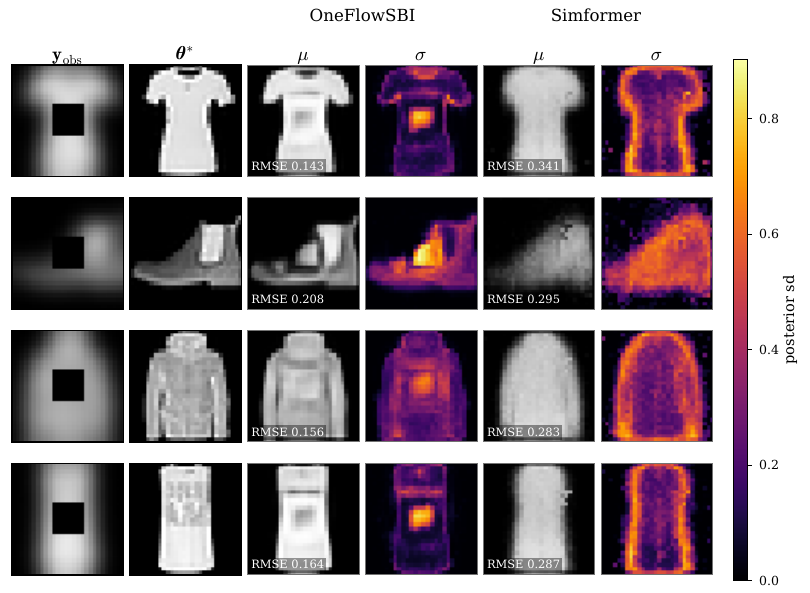}
  \caption{\textbf{Bayesian denoising: missing observation block.} Posterior mean, with its RMSE, and standard deviation ($\sigma$) given an observation whose central $8 \times 8$ block is missing, for four test images.}
  \label{fig:denoising_block_baselines}
\end{figure}

\subsection{Shallow Water: Comparison with Baselines}
\label{app:add_results_sw}

Figure~\ref{fig:sw_baselines} compares the posterior given the full observation, the one query all methods can answer, for the test pair of Figure~\ref{fig:sw} (training details in Appendix~\ref{app:hyperparam_sw}). Simformer trained on the posterior alone recovers the profile closely, with a band almost as narrow as that of \textit{OneFlowSBI} but about twice its error, and training it on all conditionals adds error and widens the band slightly. NPE recovers the coarse shape, but its posterior is much wider, and towards the far end of the channel it approaches the prior. Across test pairs, \textit{OneFlowSBI} is the most accurate: over $100$ test pairs with $100$ posterior samples each, it has the lowest RMSE of the posterior mean and the lowest energy score in each of five training runs, on average more than a third below Simformer trained on the posterior alone and about a quarter of those of NPE (Table~\ref{tab:sw_test}). Simformer sees the observation through one token per hour, compressed by a convolutional encoder, which makes the problem feasible for it; even so, with its default sampler, it samples about $24$ times slower than \textit{OneFlowSBI}, which operates on the full observation (Appendix~\ref{app:case_study_cost}).

\begin{figure}[!htbp]
  \centering
  \includegraphics[width=\columnwidth]{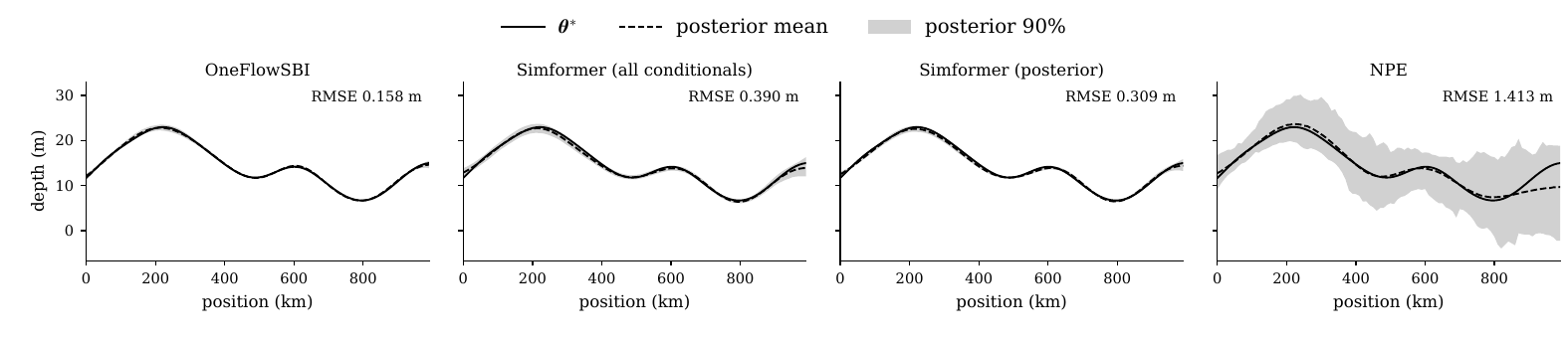}
  \caption{\textbf{Shallow water: comparison with baselines.} Posterior of
  the depth profile given the full observation, with the RMSE of the
  posterior mean, for the test pair of Figure~\ref{fig:sw}.}
  \label{fig:sw_baselines}
\end{figure}

\begin{table}[!htbp]
\centering
\small
\caption{\textbf{Shallow water: posterior accuracy on $100$ test pairs.} RMSE and energy score in metres, from $100$ samples per pair (lower is better); mean $\pm$ SD over five runs for \textit{OneFlowSBI}, mean $\pm$ SE over pairs otherwise.}
\label{tab:sw_test}
\begin{tabular}{lcc}
\toprule
Method & RMSE (m) & Energy score (m) \\
\midrule
\textit{OneFlowSBI}          & $\mathbf{0.329 \pm 0.021}$ & $\mathbf{2.41 \pm 0.18}$ \\
Simformer (all conditionals) & $0.795 \pm 0.088$ & $5.52 \pm 0.61$ \\
Simformer (posterior)        & $0.537 \pm 0.080$ & $3.75 \pm 0.55$ \\
NPE                          & $1.265 \pm 0.059$ & $9.25 \pm 0.40$ \\
\bottomrule
\end{tabular}
\end{table}

\subsubsection{Observation Windows}\label{app:sw_windows}
Table~\ref{tab:sw_windows} compares the window posteriors $p(\bm{\theta} \mid \mathbf{y}_{\mathrm{obs}}^{\le T})$ of Figure~\ref{fig:sw} with those of Simformer trained on all conditionals, which answers them by generating the later hours and discarding them, so both methods target the same posterior. Because Simformer groups the values of each hour into one token, it can answer such windows of whole hours, but not queries that condition on part of an hour, such as a subset of positions. \textit{OneFlowSBI} is more accurate at every cutoff, and its $90\%$ intervals cover the truth at close to their credibility level, whereas those of Simformer cover it too rarely, most clearly at early cutoffs, when much of the channel has not yet been reached by the wave (Figure~\ref{fig:sw_windows}).

\begin{table}[!htbp]
\centering
\small
\caption{\textbf{Shallow water: window posteriors $p(\bm{\theta} \mid \mathbf{y}_{\mathrm{obs}}^{\le T})$ on $100$ test pairs.} RMSE and energy score in metres, and $90\%$ coverage, from $100$ samples per pair; mean $\pm$ SD over five runs for \textit{OneFlowSBI}, mean $\pm$ SE over pairs for Simformer.}
\label{tab:sw_windows}
\begin{tabular}{lcccccc}
\toprule
 & \multicolumn{2}{c}{RMSE (m)} & \multicolumn{2}{c}{Energy score (m)} & \multicolumn{2}{c}{$90\%$ coverage} \\
\cmidrule(lr){2-3}\cmidrule(lr){4-5}\cmidrule(lr){6-7}
$T$ (h) & \textit{OneFlowSBI} & Simformer & \textit{OneFlowSBI} & Simformer & \textit{OneFlowSBI} & Simformer \\
\midrule
$10$ & $\mathbf{3.18 \pm 0.11}$ & $4.30 \pm 0.13$ & $\mathbf{22.44 \pm 0.85}$ & $31.80 \pm 1.07$ & $0.87$ & $0.57$ \\
$30$ & $\mathbf{0.93 \pm 0.12}$ & $1.61 \pm 0.10$ & $\mathbf{6.57 \pm 0.86}$ & $11.91 \pm 0.78$ & $0.94$ & $0.66$ \\
$60$ & $\mathbf{0.43 \pm 0.06}$ & $0.98 \pm 0.09$ & $\mathbf{3.08 \pm 0.42}$ & $6.93 \pm 0.67$ & $0.93$ & $0.81$ \\
\bottomrule
\end{tabular}
\end{table}

\begin{figure}[!htbp]
  \centering
  \includegraphics[width=\textwidth]{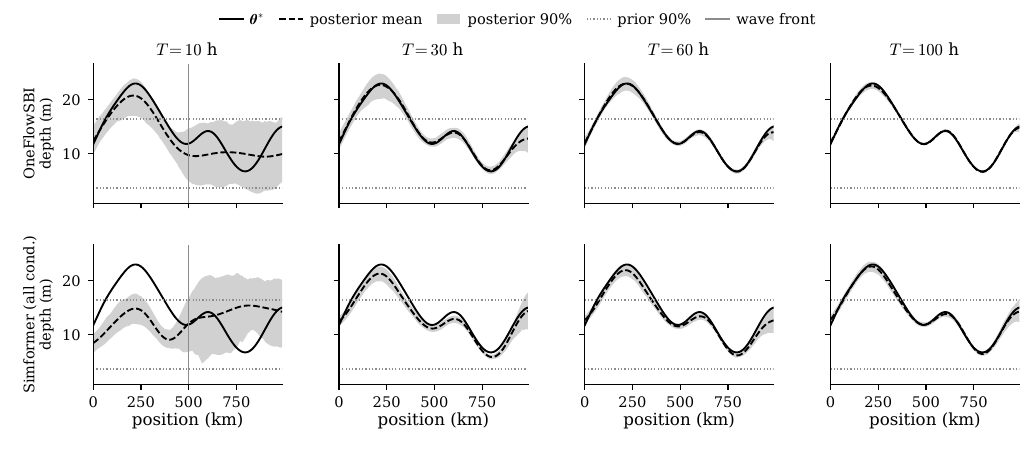}
  \caption{\textbf{Shallow water: window posteriors $p(\bm{\theta} \mid \mathbf{y}_{\mathrm{obs}}^{\le T})$ for the test pair of Figure~\ref{fig:sw}.} Posterior mean and central $90\%$ interval of the depth profile; dotted lines mark the central $90\%$ interval of the prior, and the vertical line the wave front at $T = 10$\,h. Simformer, trained on all conditionals, generates the later hours and discards them.}
  \label{fig:sw_windows}
\end{figure}

\subsection{Training and Sampling Cost in the Case Studies}
\label{app:case_study_cost}

\begin{table}[H]
\centering
\small
\caption{\textbf{Training and sampling cost.} Training time on $30{,}000$
simulations and time for $1{,}000$ posterior samples for one observation.}
\label{tab:case_study_cost}
\begin{tabular}{lcccc}
\toprule
 & \multicolumn{2}{c}{Bayesian denoising} & \multicolumn{2}{c}{Shallow water} \\
\cmidrule(lr){2-3} \cmidrule(lr){4-5}
Method & Training (min) & Sampling (s) & Training (min) & Sampling (s) \\
\midrule
\textit{OneFlowSBI}          & $39.3 \pm 0.2$ & $6.0 \pm 0.4$ & $68.7 \pm 0.4$ & $16.6 \pm 0.7$ \\
Simformer (all conditionals) & $416.5$ & $447$  & $98.6$ & $403.2$ \\
Simformer (posterior)        & $414.2$ & $443$  & $97.5$ & $395.3$ \\
NPE                          & $33.7$  & $0.08$ & $22.7$ & $0.04$  \\
\bottomrule
\end{tabular}
\end{table}

Table~\ref{tab:case_study_cost} reports the training time of each model on $30{,}000$ simulations and the time to draw $1{,}000$ posterior samples for one observation. NPE samples in a single pass and is fastest, but learns only the posterior. Among the models that learn all conditionals, \textit{OneFlowSBI} trains faster than Simformer and, even with the adaptive \texttt{dopri5} solver, which evaluates the network six times per step at tight tolerances, samples one to two orders of magnitude faster than Simformer with its $500$ Euler--Maruyama steps.

\end{document}